\documentclass[letterpaper]{article} 
\usepackage{aaai2027}  
\usepackage[hyphens]{url}  
\usepackage{graphicx} 
\usepackage{natbib}  
\usepackage{caption} 
\usepackage{algorithm}
\usepackage{algorithmic}

\usepackage{newfloat}
\usepackage{listings}
\DeclareCaptionStyle{ruled}{labelfont=normalfont,labelsep=colon,strut=off} 
\floatstyle{ruled}
\newfloat{listing}{tb}{lst}{}
\floatname{listing}{Listing}

\usepackage{booktabs}

\usepackage{algorithm}
\usepackage{algorithmic}
\usepackage{multirow}
\usepackage[table]{xcolor}
\usepackage{url}
\usepackage{graphicx}
\usepackage{subcaption}
\usepackage{amsmath}
\usepackage{algorithm}
\usepackage{algorithmic}
\usepackage{amssymb}
\usepackage{mathtools}
\usepackage{amsthm}
\usepackage{booktabs}
\usepackage[table]{xcolor}
\usepackage[most]{tcolorbox}
\usepackage{xcolor}

\theoremstyle{plain}
\theoremstyle{plain}
\newtheorem{theorem}{Theorem}

\theoremstyle{definition}

\theoremstyle{remark}

\definecolor{outputtitlegray}{RGB}{73,71,71}
\definecolor{questionblue}{RGB}{20,20,220}
\definecolor{correctgreen}{RGB}{92,166,55}

\newtcolorbox{modeloutputbox}[1]{
    enhanced,
    width=\linewidth,
    colback=white,
    colframe=black,
    colbacktitle=outputtitlegray,
    coltitle=white,
    fonttitle=\sffamily\large,
    title={#1},
    boxrule=0.8pt,
    titlerule=0.8pt,
    arc=3mm,
    outer arc=3mm,
    left=2mm,
    right=2mm,
    top=1.5mm,
    bottom=1.5mm,
    toptitle=1.5mm,
    bottomtitle=1.5mm,
    before skip=2mm,
    after skip=1mm
}

\definecolor{prompttitlegray}{RGB}{73,71,71}

\newtcolorbox{promptbox}[1]{
    enhanced,
    width=\linewidth,
    colback=white,
    colframe=black,
    colbacktitle=prompttitlegray,
    coltitle=white,
    fonttitle=\sffamily\bfseries,
    title={#1},
    boxrule=0.8pt,
    titlerule=0.8pt,
    arc=2.5mm,
    outer arc=2.5mm,
    left=2.5mm,
    right=2.5mm,
    top=2mm,
    bottom=2mm,
    toptitle=1.5mm,
    bottomtitle=1.5mm,
    before skip=1.5mm,
    after skip=1.5mm
}

\theoremstyle{plain}

\title{Toward Plasticity-Preserving KL Regularization for Capability Retention in LLM Reinforcement Learning}
\author{%
\begin{tabular}{c}
Li Wang$^{1,}$\thanks{Equal contribution.} \quad
Xiaodong Lu$^{1,}$\footnotemark[1] \quad
Xiaohan Wang$^{1,}$\thanks{%
\raggedright
Corresponding authors:
\mbox{\texttt{tianhao.peng@ntu.edu.sg}};\protect\newline
\mbox{\texttt{\{wangxiaohan17,yinguojun02\}@meituan.com}}.%
} \quad
Jiajun Chai$^{1}$ \\
Wei Lin$^{1}$ \quad
Tianhao Peng$^{2,}$\footnotemark[2] \quad
Guojun Yin$^{1,}$\footnotemark[2]
\end{tabular}%
}

\affiliations{
\textsuperscript{\rm 1}Meituan\\
\textsuperscript{\rm 2}Nanyang Technological University
}

\begin{document}

\maketitle

\begin{abstract}
Reinforcement learning (RL) has become a central paradigm for large language model (LLM) post-training, but optimization toward new objectives can degrade capabilities already present in the base model. KL regularization is widely used to mitigate such forgetting by constraining policy drift toward a reference model. However, standard full-policy KL regularization constrains the entire response distribution and may unnecessarily restrict exploration and target-task learning. This raises a natural question: can a more precise constraint preserve existing capabilities while minimizing interference with learning new tasks? To this end, we propose \underline{Co}rrectness-Conditioned \underline{KL} Regularization (CoKL), a conditional regularization framework that narrows the preservation constraint from the full output distribution to correctness-conditioned response distributions. We instantiate CoKL with forward KL divergence and derive a practical finite-group training objective for RL-based LLM post-training. At the population level, CoKL decouples the total probability
assigned to correct responses from their correctness-conditioned
distribution, thereby regularizing the relative probability
allocation among reference-supported correct responses without
directly anchoring incorrect outputs or total correctness mass. We further show that full-policy forward and reverse KL regularization induce a strict optimal correctness gap when the reference policy is imperfect, whereas CoKL avoids this limitation. Experiments in controlled multi-solution environments and continual post-training settings across multiple model scales demonstrate that CoKL achieves a more favorable balance between target-task improvement and prior-capability retention than existing regularization methods. Our code is available at \url{https://github.com/Lumina04/CoKL}.
\end{abstract}


\section{Introduction}
Reinforcement learning (RL) has become a central paradigm for large language model (LLM) post-training. Reinforcement Learning from Human Feedback (RLHF)~\cite{bai2022training,ouyang2022training} is widely used to align model behavior with human preferences, while Reinforcement Learning with Verifiable Rewards (RLVR)~\cite{guo2025deepseek} enables scalable optimization of reasoning performance through automatically verifiable feedback~\cite{lin2026resrl,lu2026contextual}. Despite these advances, RL-based post-training can degrade capabilities already encoded in the base model. By repeatedly amplifying reward-favored outputs, RL may concentrate probability mass on a narrow set of high-reward behaviors while suppressing alternative response patterns and reasoning strategies that remain valid. This form of capability forgetting can manifest as degraded performance on previously learned or out-of-distribution tasks~\cite{lin2024mitigating,kotha2024understanding,cai2026advancing,li2025omni}, indicating that adaptation to a new objective
may interfere with capabilities acquired before post-training. This issue is particularly consequential for general-purpose LLMs, for which post-training should improve adaptation to new objectives without compromising previously acquired capabilities. Achieving effective adaptation while preserving existing capabilities therefore remains a fundamental challenge in RL-based LLM post-training.

\begin{figure}[t]
    \centering
    \includegraphics[width=0.45\textwidth]{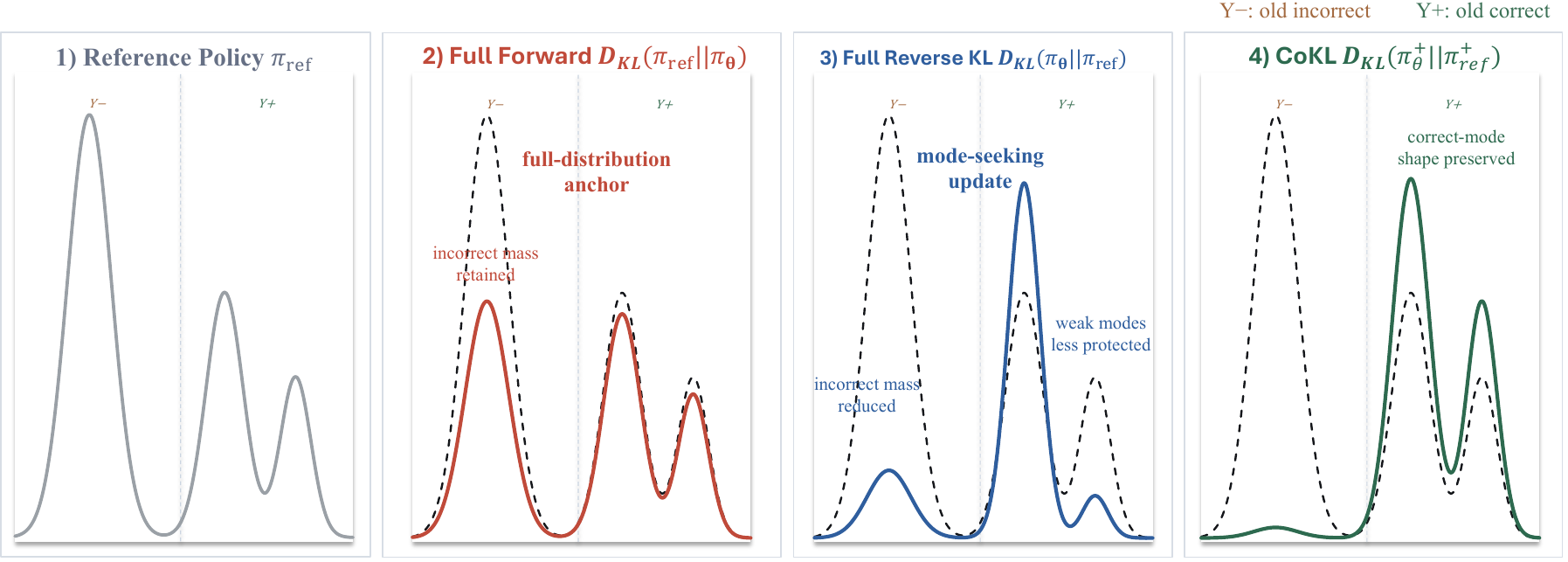}
    \caption{Schematic comparison of KL constraint scopes. Full-policy KL constrains correctness mass and the conditional distributions of both correct and incorrect responses, whereas CoKL constrains only the correctness-conditioned distribution without anchoring incorrect responses or total correctness mass.}    \label{fig:kl_regularizer_comparison}
\end{figure}

KL regularization is widely used to preserve pretrained capabilities during LLM post-training by constraining policy drift from the base model. As shown in Figure~\ref{fig:kl_regularizer_comparison}, reverse KL tends to favor high-probability regions of the reference policy, whereas forward KL more strongly preserves its support. However, existing KL regularizers are typically applied to the full response distribution\cite{shao2024deepseekmath,liu2025understanding}. Such global constraints may preserve useful behaviors, but can also retain ineffective or erroneous modes, thereby restricting exploration and adaptation to new tasks. To avoid this limitation, some approaches remove the KL penalty entirely to encourage exploration~\cite{yu2026dapo,xue2025simpletir}. Yet unconstrained policy updates can cause substantial drift from the base model and increase the risk of capability forgetting. This raises a natural question: can we impose constraints only where preservation is necessary, thereby retaining existing capabilities while minimizing interference with adaptation during post-training?

To preserve existing capabilities without unnecessarily restricting learning during post-training, we propose \underline{Co}rrectness-Conditioned \underline{KL} Regularization (CoKL), a more precise conditional regularization framework that constrains the response distribution according to correctness feedback rather than the full output distribution. In this work, we instantiate CoKL with forward KL divergence and apply it to the conditional distribution over verified-correct responses. CoKL preserves the relative probability allocation among correct response modes inherited from the reference policy while allowing the total probability assigned to correct responses to increase. Incorrect response modes are therefore not anchored to the reference policy and can be freely reshaped by the post-training objective. By narrowing the scope of regularization rather than uniformly reducing its strength, CoKL retains previously accessible correct behaviors while minimizing interference with the acquisition of new capabilities. We derive a practical finite-group objective for RL-based LLM post-training and provide a population-level analysis showing that CoKL decouples the total correctness probability from the conditional distribution over correct responses. In contrast, full-policy forward and reverse KL regularization couple these quantities and induce a strict optimal correctness gap when the reference policy is imperfect. Controlled simulations and continual post-training experiments across multiple model scales further demonstrate that CoKL achieves a favorable balance between capability retention and learning during post-training. Our contributions are summarized as follows:

\begin{itemize}
\item We introduce a more precise preservation principle for LLM post-training. Rather than constraining the full response distribution, it restricts regularization to the conditional distribution over verified-correct responses, avoiding direct constraints on incorrect outputs and total correctness mass.

\item We instantiate this principle as CoKL, a practical conditional regularization method for RL-based LLM post-training, and derive its finite-group training objective. Our population-level analysis
shows that CoKL decouples total correctness probability from
the correctness-conditioned response distribution, whereas
full-policy KL induces a strict optimal correctness gap under
an imperfect reference policy.

\item We evaluate the distributional behavior of CoKL in a
controlled multi-solution environment and assess its practical
retention--adaptation trade-off in continual LLM post-training
across multiple model scales.
\end{itemize}

\section{Related work}
\subsection{LLM RL and Capability Preservation}
Reinforcement learning has become a central component of LLM post-training, enabling models to optimize task-specific objectives and improve complex reasoning capabilities. Building on policy optimization methods such as GRPO~\cite{shao2024deepseekmath}, recent algorithms including DAPO~\cite{yu2026dapo} and GSPO~\cite{zheng2025group} improve training stability and optimization efficiency, yielding substantial gains in mathematical reasoning~\cite{lin2026resrl,lu2026contextual} and tool use~\cite{lin2025rest,wang2026implicit}. However, optimization toward new objectives can shift the model away from its original policy distribution, degrading performance on previously learned tasks~\cite{li2025can} or suppressing effective reasoning patterns already encoded in the base model~\cite{chen2026does}. Prior work has mitigated such interference through regularization, parameter merging~\cite{wang2026mix}, and multi-domain training~\cite{cai2026advancing}. In contrast, we revisit KL regularization in RL-based post-training and introduce a simple yet more precise constraint that preserves existing capabilities while reducing interference with post-training learning.

\subsection{KL Regularization in LLM RL}
KL regularization is a standard component in RLHF and RLVR for stabilizing policy optimization and preventing excessive deviation from a reference policy~\cite{ziegler2019fine,stiennon2020learning,ouyang2022training}. Existing methods typically apply KL regularization to the full policy distribution while adjusting its direction, strength, or token-level application to balance training stability and exploration~\cite{yu2026dapo,xue2025simpletir,deng2025unlocking,li2025choice,lin2026expo,wang2025stabilizing,cui2025entropy,vassoyan2025ignore,lee2026sage,gx2025kl}. However, full-policy constraints do not distinguish the total probability assigned to correct responses from the relative probability allocation within the correct-response set. Recent verifier-aware methods instead operate on correct responses to improve distributional coverage or diversity. Kruszewski et al.~\cite{kruszewski2025whatever} construct an explicit target distribution by filtering incorrect responses while preserving the relative reference probabilities among correct responses, UCPO~\cite{lochab2026uniform} encourages a uniform allocation over correct solutions, and DSDR~\cite{wan2026dsdr} introduces trajectory- and token-level regularization to promote diversity among correct reasoning paths. In contrast, CoKL conditions both the reference and current policies on correctness and regularizes $D_{\mathrm{KL}}(\pi_{\mathrm{ref}}^{+}\Vert\pi_{\theta}^{+})$. Unlike correct-only target matching, this preserves the reference-induced allocation within the correct-response set without directly optimizing total correctness mass or constraining incorrect responses.


\section{Preliminaries}
\subsection{Group Relative Policy Optimization}
Group Relative Policy Optimization (GRPO)~\cite{shao2024deepseekmath}
eliminates the need for a learned value function by estimating advantages
from multiple responses generated for the same prompt. Given a prompt
$x_i$, the behavior policy $\pi_{\theta_{\mathrm{old}}}$ samples a group
$\mathcal{O}_i=\{o_{i,g}\}_{g=1}^{G}$, where each response receives a
verifier reward $r_{i,g}$. Its group-normalized advantage is
\begin{equation}
A_{i,g}
=
\frac{r_{i,g}-\mu_i}{\sigma_i+\epsilon},
\qquad
\mu_i=\frac{1}{G}\sum_{g=1}^{G}r_{i,g},
\label{eq:grpo_adv}
\end{equation}
where $\sigma_i$ is the standard deviation of rewards within
$\mathcal{O}_i$, and $\epsilon$ is a small numerical constant. Since the
reward is assigned at the response level, $A_{i,g}$ is shared by all
tokens in $o_{i,g}$. Let
\begin{equation}
\rho_{i,g,t}(\theta)
=
\frac{
\pi_\theta(o_{i,g,t}\mid x_i,o_{i,g,<t})
}{
\pi_{\theta_{\mathrm{old}}}
(o_{i,g,t}\mid x_i,o_{i,g,<t})
},
\label{eq:grpo_ratio}
\end{equation}
and define the clipped surrogate
\begin{equation}
\ell_{i,g,t}(\theta)
=
\min\left\{
\rho_{i,g,t}A_{i,g},
\operatorname{clip}
\bigl(\rho_{i,g,t},1-\delta,1+\delta\bigr)A_{i,g}
\right\}.
\label{eq:grpo_surrogate}
\end{equation}
The GRPO objective is then written as
\begin{equation}
\mathcal{J}_{\mathrm{GRPO}}(\theta)
=
\mathbb{E}
\left[
\frac{1}{G}
\sum_{g=1}^{G}
\frac{1}{|o_{i,g}|}
\sum_{t=1}^{|o_{i,g}|}
\ell_{i,g,t}(\theta)
\right],
\label{eq:grpo_obj}
\end{equation}
where $\delta$ controls the clipping range.

\section{Method}
To preserve prior capabilities while minimizing interference with new-task learning, we propose \underline{Co}rrectness-Conditioned \underline{KL} Regularization (CoKL), which constrains only the distribution over verified-correct responses. We derive its forward-KL objective and finite-group surrogate and integrate it into GRPO.

\subsection{Correctness Decomposition of Full-Policy KL}
For clarity, we formulate CoKL using binary verifier feedback. For each
prompt \(x\), let \(r(x,y)\in\{0,1\}\) indicate whether response \(y\)
is accepted as correct by the verifier, thereby inducing a binary
partition of the response space. When the verifier provides a continuous
score \(s(x,y)\), an analogous partition can be obtained using a
task-dependent threshold \(\tau\), with
\(r_{\tau}(x,y)=\mathbb{I}[s(x,y)\geq\tau]\). For any policy \(\pi\), we define its total
correctness probability as
\begin{equation}
Z_\pi(x)
=
\sum_y r(x,y)\pi(y\mid x),
\label{eq:correctness-probability}
\end{equation}
and define the corresponding conditional distributions as
\begin{equation}
\begin{aligned}
\pi^+(y\mid x)
&=
\frac{r(x,y)\pi(y\mid x)}
{Z_\pi(x)},\\
\pi^-(y\mid x)
&=
\frac{\bigl(1-r(x,y)\bigr)\pi(y\mid x)}
{1-Z_\pi(x)}.
\end{aligned}
\label{eq:correct-conditioned-policy}
\end{equation}
Here, \(\pi^+\) and \(\pi^-\) denote the policy conditioned on producing
a correct or incorrect response, respectively. Let \(\pi_{\mathrm{ref}}\) denote the frozen reference policy and
\(\pi_\theta\) the current policy. For readability, we omit the prompt
\(x\) below and write
\(Z_{\mathrm{ref}}=Z_{\pi_{\mathrm{ref}}}(x)\) and
\(Z_\theta=Z_{\pi_\theta}(x)\).
The full-policy forward KL admits the following decomposition:
\begin{equation}
\begin{aligned}
D_{\mathrm{KL}}
\left(
\pi_{\mathrm{ref}}
\|
\pi_\theta
\right)
=&\;
D_{\mathrm{KL}}
\left(
\operatorname{Bern}(Z_{\mathrm{ref}})
\|
\operatorname{Bern}(Z_\theta)
\right)
\\
&+
Z_{\mathrm{ref}}
D_{\mathrm{KL}}
\left(
\pi_{\mathrm{ref}}^+
\|
\pi_\theta^+
\right)
\\
&+
(1-Z_{\mathrm{ref}})
D_{\mathrm{KL}}
\left(
\pi_{\mathrm{ref}}^-
\|
\pi_\theta^-
\right).
\end{aligned}
\label{eq:kl-correctness-decomposition}
\end{equation}
The derivation is in Appendix~A. Eq.~\ref{eq:kl-correctness-decomposition} shows that full-policy KL
simultaneously constrains the total probability of correctness, the relative
distribution among correct responses, and the distribution among incorrect
responses. The first term pulls the current correctness probability $Z_\theta$ toward that of $\pi_{\mathrm{ref}}$, which may conflict with the RL-based post-training objective when the reference policy initially assigns limited probability mass to correct responses. The final term further anchors the policy to incorrect modes. These constraints are not necessary for preserving
effective solutions and may restrict the transfer of probability mass from
incorrect to correct responses. We therefore retain only the
correctness-conditioned component as the preservation target.

\subsection{CoKL Objective and Gradient Interpretation}
Motivated by the decomposition above, we define CoKL as the forward KL
divergence between the correctness-conditioned reference policy and the
correctness-conditioned current policy:
\begin{equation}
\mathcal{L}_{\mathrm{CoKL}}(x)
=
D_{\mathrm{KL}}\left(
\pi_{\mathrm{ref}}^+(\cdot\mid x)
\|
\pi_\theta^+(\cdot\mid x)
\right).
\label{eq:cfkl-def}
\end{equation}
Unlike a full-policy KL regularizer, CoKL compares the two policies only
after conditioning on correctness. It therefore penalizes changes in the relative probability
allocation among correct responses, without explicitly
anchoring incorrect responses under the reference policy. For a correct response \(y\), we have
\(\log \pi_\theta^+(y\mid x)
=
\log \pi_\theta(y\mid x)-\log Z_\theta(x)\).
Therefore, up to terms independent of \(\theta\),
\begin{equation}
\mathcal{L}_{\mathrm{CoKL}}(x)
\doteq
-
\mathbb{E}_{y\sim\pi_{\mathrm{ref}}^+(\cdot\mid x)}
\left[
\log \pi_\theta(y\mid x)
\right]
+
\log Z_\theta(x).
\label{eq:cfkl-opt-form}
\end{equation}
The \(\log Z_\theta(x)\) term removes the implicit correctness-mass
optimization in correct-only forward-KL replay:
\[
D_{\mathrm{KL}}\!\left(
\pi_{\mathrm{ref}}^+ \Vert \pi_\theta
\right)
=
D_{\mathrm{KL}}\!\left(
\pi_{\mathrm{ref}}^+ \Vert \pi_\theta^+
\right)
-\log Z_\theta.
\]
Thus, CoKL preserves the relative distribution among correct responses while leaving the total correctness probability to be optimized by the RL-based post-training objective. Its gradient is
\begin{equation}
\begin{aligned}
\nabla_\theta \mathcal{L}_{\mathrm{CoKL}}(x)
=&
-
\mathbb{E}_{y\sim\pi_{\mathrm{ref}}^+(\cdot\mid x)}
\left[
\nabla_\theta\log\pi_\theta(y\mid x)
\right]
\\
&+
\mathbb{E}_{y\sim\pi_\theta^+(\cdot\mid x)}
\left[
\nabla_\theta\log\pi_\theta(y\mid x)
\right].
\end{aligned}
\label{eq:cfkl-gradient}
\end{equation}
The first term anchors the update toward reference-supported
correct responses, while the second provides the
normalization correction required by conditioning. The full derivation are provided in Appendix~B.

\subsection{Sample-based CoKL Surrogate}
The distribution-level objective in Eq.~\ref{eq:cfkl-def} is approximated during RL-based LLM post-training using sampled responses and verifier outcomes. For each prompt group, we construct self-normalized empirical
weights over the verified-correct responses. We first present
the on-policy finite-group surrogate, where both the
reference-side and current-policy-side terms are estimated
from fresh samples. Let \(\{y_{ij}\}_{j=1}^G\sim\pi_{\mathrm{ref}}(\cdot\mid x_i)\) be reference samples with rewards \(r_{ij}=r(x_i,y_{ij})\), and let
\(\{y'_{ij}\}_{j=1}^G\sim\pi_\theta(\cdot\mid x_i)\) be current-policy samples with rewards \(r'_{ij}=r(x_i,y'_{ij})\). We define the empirical conditional weights as
\begin{equation}
\alpha_{ij}
=
\frac{r_{ij}}
{\sum_{k=1}^G r_{ik}},
\qquad
\delta_{ij}
=
\frac{r'_{ij}}
{\sum_{k=1}^G r'_{ik}}.
\label{eq:on-policy-weights}
\end{equation}
with the convention that \(\alpha_{ij}=0\) if
\(\sum_{k=1}^G r_{ik}=0\), and \(\delta_{ij}=0\) if
\(\sum_{k=1}^G r'_{ik}=0\). The on-policy CoKL surrogate is
\begin{equation}
\begin{aligned}
\widehat{\mathcal{L}}_{\mathrm{CoKL}}
=&
-\frac{1}{N}
\sum_{i=1}^{N}
\sum_{j=1}^G
\operatorname{sg}(\alpha_{ij})
\log \pi_\theta(y_{ij}\mid x_i)
\\
&+
\frac{1}{N}
\sum_{i=1}^{N}
\sum_{j=1}^G
\operatorname{sg}(\delta_{ij})
\log \pi_\theta(y'_{ij}\mid x_i).
\end{aligned}
\label{eq:cfkl-on-policy-loss}
\end{equation}
where \(\operatorname{sg}(\cdot)\) denotes stop-gradient and \(N\) is
the number of prompt groups in the minibatch. When a current-policy
group contains verified-correct responses, the two terms approximate
the reference- and current-policy expectations in
Eq.~\ref{eq:cfkl-gradient}, respectively. If no verified-correct response
is observed, the current-policy correction is unavailable and the
surrogate reduces to the reference-correct term. Since such zero-correct
groups occur more frequently when the current correctness probability is
low, this fallback acts as an implicit adaptive gate that strengthens
the recovery pressure toward correct responses, as formalized in
Appendix~B.

In the off-policy setting, the current-policy conditional term is estimated from responses generated by a behavior policy \(\pi_{\mathrm{old}}\). Let
\(\{\hat y_{ij}\}_{j=1}^G\sim\pi_{\mathrm{old}}(\cdot\mid x_i)\) denote the rollout group, with verifier rewards
\(\hat r_{ij}=r(x_i,\hat y_{ij})\). Since CoKL is defined over complete responses, we apply importance correction at the sequence level. For a response \(\hat y_{ij}\) of length \(T_{ij}\), we write its sequence log-likelihood as
\begin{equation}
\ell_\theta(\hat y_{ij}\mid x_i)
=
\sum_{t=1}^{T_{ij}}
\log \pi_\theta(
\hat y_{ij,t}\mid x_i,\hat y_{ij,<t}
).
\label{eq:sequence-logprob}
\end{equation}
The sequence-level log-ratio is then
\begin{equation}
\Delta_{ij}(\theta)
=
\ell_\theta(\hat y_{ij}\mid x_i)
-
\ell_{\mathrm{old}}(\hat y_{ij}\mid x_i).
\label{eq:log-importance-ratio}
\end{equation}
To control the variance of off-policy correction, we use the clipped sequence-level importance weight
\begin{equation}
\tilde\rho_{ij}
=
\exp\left[
\operatorname{clip}
\left(
\Delta_{ij}(\theta),
\log(1-\epsilon_{\mathrm{IS}}),
\log(1+\epsilon_{\mathrm{IS}})
\right)
\right].
\label{eq:clipped-ratio}
\end{equation}
The empirical current conditional distribution is approximated by the self-normalized weights
\begin{equation}
\gamma_{ij}
=
\frac{\tilde\rho_{ij}\hat r_{ij}}
{\sum_{k=1}^G \tilde\rho_{ik}\hat r_{ik}}.
\label{eq:off-policy-weight}
\end{equation}
If the denominator is zero, we set \(\gamma_{ij}=0\) for all responses in that group. The clipped off-policy CoKL surrogate is then
\begin{equation}
\begin{aligned}
\widehat{\mathcal{L}}_{\mathrm{CoKL}}
=&
-
\frac{1}{N}
\sum_{i=1}^{N}
\sum_{j=1}^G
\operatorname{sg}(\alpha_{ij})
\log \pi_\theta(y_{ij}\mid x_i)
\\
&+
\frac{1}{N}
\sum_{i=1}^{N}
\sum_{j=1}^G
\operatorname{sg}(\gamma_{ij})
\log \pi_\theta(\hat y_{ij}\mid x_i).
\end{aligned}
\label{eq:cfkl-practical-loss}
\end{equation}

\subsection{Integration with RL-Based LLM Post-Training}
We integrate CoKL into RL-based LLM post-training through a reference-correct buffer and a joint RL and regularization objective. Before training, \(\pi_{\mathrm{ref}}\) is used to generate \(G_{\mathrm{ref}}\) responses for each prompt in the regularization prompt set. The verifier is then applied to these responses, and prompts with no verified correct reference response are filtered out. The remaining prompt groups form a reference-correct buffer \(\mathcal{D}^+\), which provides the empirical support for the reference-side term in Eq.~\ref{eq:cfkl-practical-loss}. During training, we sample a current-task minibatch \(\mathcal{B}_{\mathrm{RL}}\) from the current-task data and independently sample a regularization minibatch \(\mathcal{B}_{\mathrm{KL}}\) from \(\mathcal{D}^+\). Current-task rollouts are used to compute \(\mathcal{L}_{\mathrm{GRPO}}\), whereas rollouts generated for \(\mathcal{B}_{\mathrm{KL}}\) instantiate the current-policy conditional term in Eq.~\ref{eq:cfkl-practical-loss}. The policy is optimized by minimizing
\begin{equation}
\begin{aligned}
\mathcal{L}_{\mathrm{total}}(\theta)
&=
\mathcal{L}_{\mathrm{GRPO}}(\theta)
+
\beta\widehat{\mathcal{L}}_{\mathrm{CoKL}}(\theta),\\
\mathcal{L}_{\mathrm{GRPO}}(\theta)
&=
-J_{\mathrm{GRPO}}(\theta).
\end{aligned}
\label{eq:total-objective}
\end{equation}
where \(J_{\mathrm{GRPO}}\) is defined in Eq.~4,
\(\widehat{\mathcal{L}}_{\mathrm{CoKL}}\) is the clipped
off-policy CoKL surrogate, and \(\beta\) controls the strength of
correctness-conditioned preservation. Minimizing
\(\mathcal{L}_{\mathrm{GRPO}}\) increases the probability of verified-correct responses, while \(\widehat{\mathcal{L}}_{\mathrm{CoKL}}\) preserves the relative probability allocation among reference-supported correct modes. The combined objective improves the current task while
preserving the relative allocation among reference-supported
correct modes on the regularization prompts, without explicitly
constraining their total correctness mass or incorrect responses. The full algorithm is summarized in Appendix~I.

\subsection{Theoretical Analysis}
\label{sec:theoretical-analysis}
We next compare the population-level optima induced by CoKL
and full-policy KL regularization. To isolate the objective-level bias introduced by different
regularizers, we first consider the setting where correctness
optimization and regularization are defined on the same prompt
distribution.

\begin{theorem}[Correctness-Mass Decoupling under CoKL]
\label{thm:cokl_decoupling}
For $\beta>0$, consider the population-level objective
\[
J_{\mathrm{CoKL}}(\pi)
=
Z_\pi
-
\beta D_{\mathrm{KL}}
\left(
\pi_{\mathrm{ref}}^+
\Vert
\pi^+
\right).
\]
The objective depends on $\pi$ only through $Z_\pi$ and
$\pi^+$ and is independent of $\pi^-$. For any two policies
$\pi_1$ and $\pi_2$ satisfying
$\pi_1^+=\pi_2^+$ and $Z_{\pi_2}>Z_{\pi_1}$,
\[
J_{\mathrm{CoKL}}(\pi_2)
-
J_{\mathrm{CoKL}}(\pi_1)
=
Z_{\pi_2}-Z_{\pi_1}
>
0.
\]
Moreover, over the full probability simplex,
$\max_\pi J_{\mathrm{CoKL}}(\pi)=1$, and every global
maximizer satisfies
\[
Z_{\pi^\star}=1,
\qquad
\pi^{+\star}=\pi_{\mathrm{ref}}^+.
\]
\end{theorem}

The full proof is provided in Appendix C. Theorem~\ref{thm:cokl_decoupling} shows that, under the population-level CoKL
objective, increasing the total probability assigned to correct
responses does not increase the regularization penalty when
their relative allocation is unchanged. The objective therefore
permits probability mass to move from incorrect to correct
responses while preserving the correct-response structure
inherited from the reference policy. Full-policy KL behaves differently because its
Bernoulli component directly constrains the probability
allocation between the correct and incorrect response sets.
The following theorem quantifies the resulting correctness
gap at the population-level optimum.

\begin{theorem}[Strict Correctness Gap Induced by Full-Policy KL]
\label{thm:full_kl_gap}
Let $q=Z_{\mathrm{ref}}\in(0,1)$ and $\beta>0$. Consider
the full-policy forward- and reverse-KL objectives
\[
\begin{aligned}
J_{\mathrm{FKL}}(\pi)
&=
Z_\pi
-
\beta D_{\mathrm{KL}}
\left(
\pi_{\mathrm{ref}}\Vert\pi
\right),
\\
J_{\mathrm{RKL}}(\pi)
&=
Z_\pi
-
\beta D_{\mathrm{KL}}
\left(
\pi\Vert\pi_{\mathrm{ref}}
\right).
\end{aligned}
\]
Under the required absolute-continuity conditions, their
global optimizers satisfy
$\pi^{+\star}=\pi_{\mathrm{ref}}^+$ and
$\pi^{-\star}=\pi_{\mathrm{ref}}^-$. The corresponding
optimal correctness probabilities are
\[
\begin{aligned}
Z_{\mathrm{FKL}}^\star
&=
\frac{
1-\beta+
\sqrt{(\beta-1)^2+4\beta q}
}{2},
\\
Z_{\mathrm{RKL}}^\star
&=
\sigma\left(
\operatorname{logit}(q)+\frac{1}{\beta}
\right),
\end{aligned}
\]
where $\sigma(a)=(1+e^{-a})^{-1}$. Both satisfy
$q<Z^\star<1$ and are strictly decreasing in $\beta$.
Furthermore, for
$\diamond\in\{\mathrm{FKL},\mathrm{RKL}\}$,
\[
\lim_{\beta\to0^+}Z_\diamond^\star=1,
\qquad
\lim_{\beta\to\infty}Z_\diamond^\star=q.
\]
\end{theorem}

The proof is provided in Appendix D. Together, Theorems 1 and 2
establish a fundamental distinction between the population-level
objectives. The distribution-level CoKL objective preserves the
conditional structure of verified-correct responses without
penalizing their total probability, whereas full-policy KL shifts
the global optimum away from the pure-reward solution
$Z_\pi=1$ whenever $\pi_{\mathrm{ref}}$ is imperfect. These results
characterize the intrinsic objective-level bias of the regularizers,
while their practical effects under shared neural parameterization
are evaluated in the continual post-training experiments.

\section{Experiments}
We evaluate CoKL at two complementary levels. The controlled
environment provides direct access to correctness-conditioned
distributions and correct-mode coverage, while the continual
LLM experiments assess the downstream trade-off between
prior-task retention and new-task adaptation.

\begin{figure}[t]
\centering

\begin{minipage}[t]{0.22\textwidth}
    \centering
    \includegraphics[width=\linewidth]
    {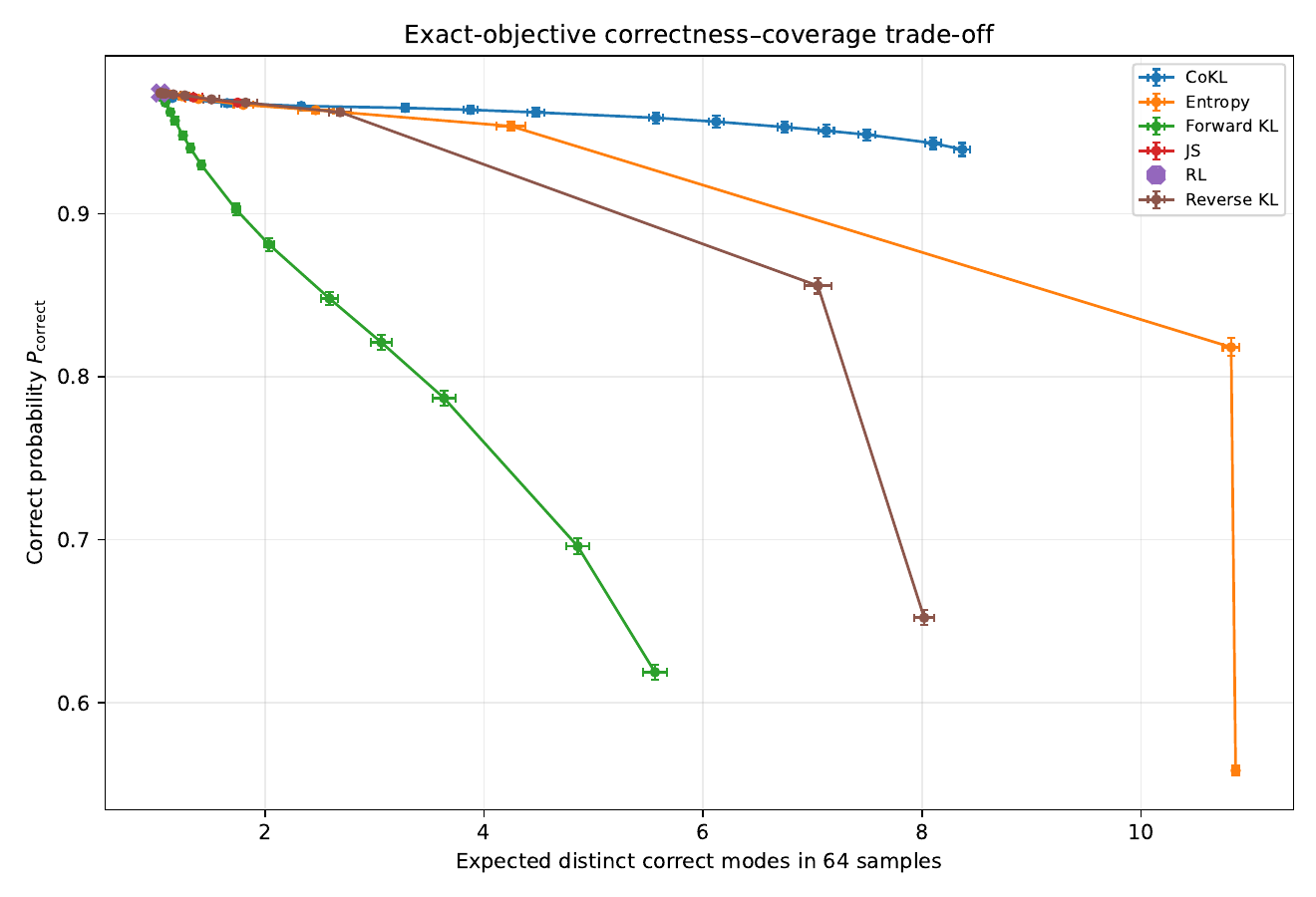}
\end{minipage}
\begin{minipage}[t]{0.22\textwidth}
    \centering
    \includegraphics[width=\linewidth]
    {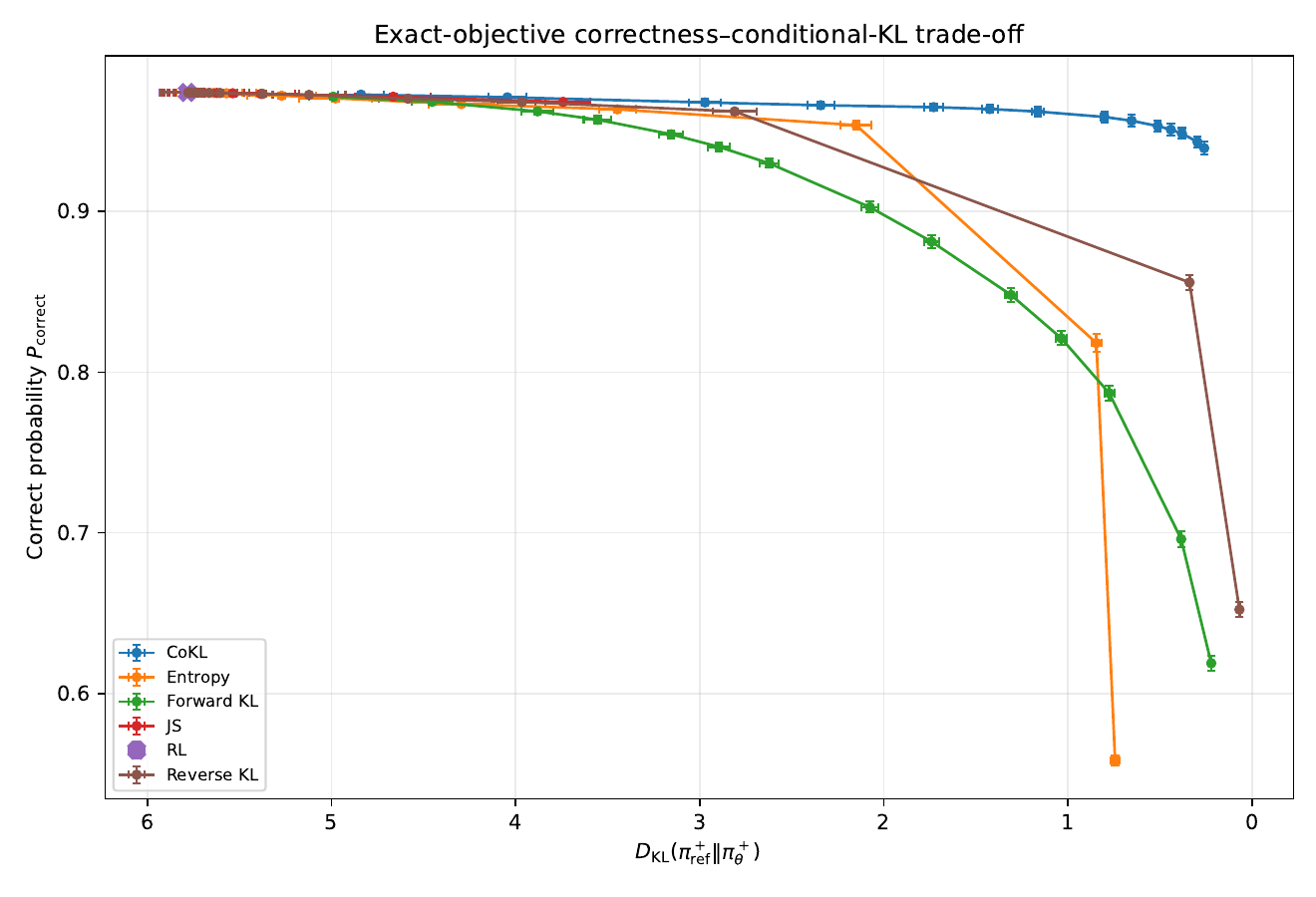}
\end{minipage}

\caption{Trade-offs between correctness and coverage@64 (Left), and between correctness and conditional KL (Right).}
\label{fig:controlled_multi_solution_results_tradeoff}
\end{figure}

\begin{figure}[t]
\centering

\begin{minipage}[t]{0.23\textwidth}
    \centering
    \includegraphics[width=\linewidth]
    {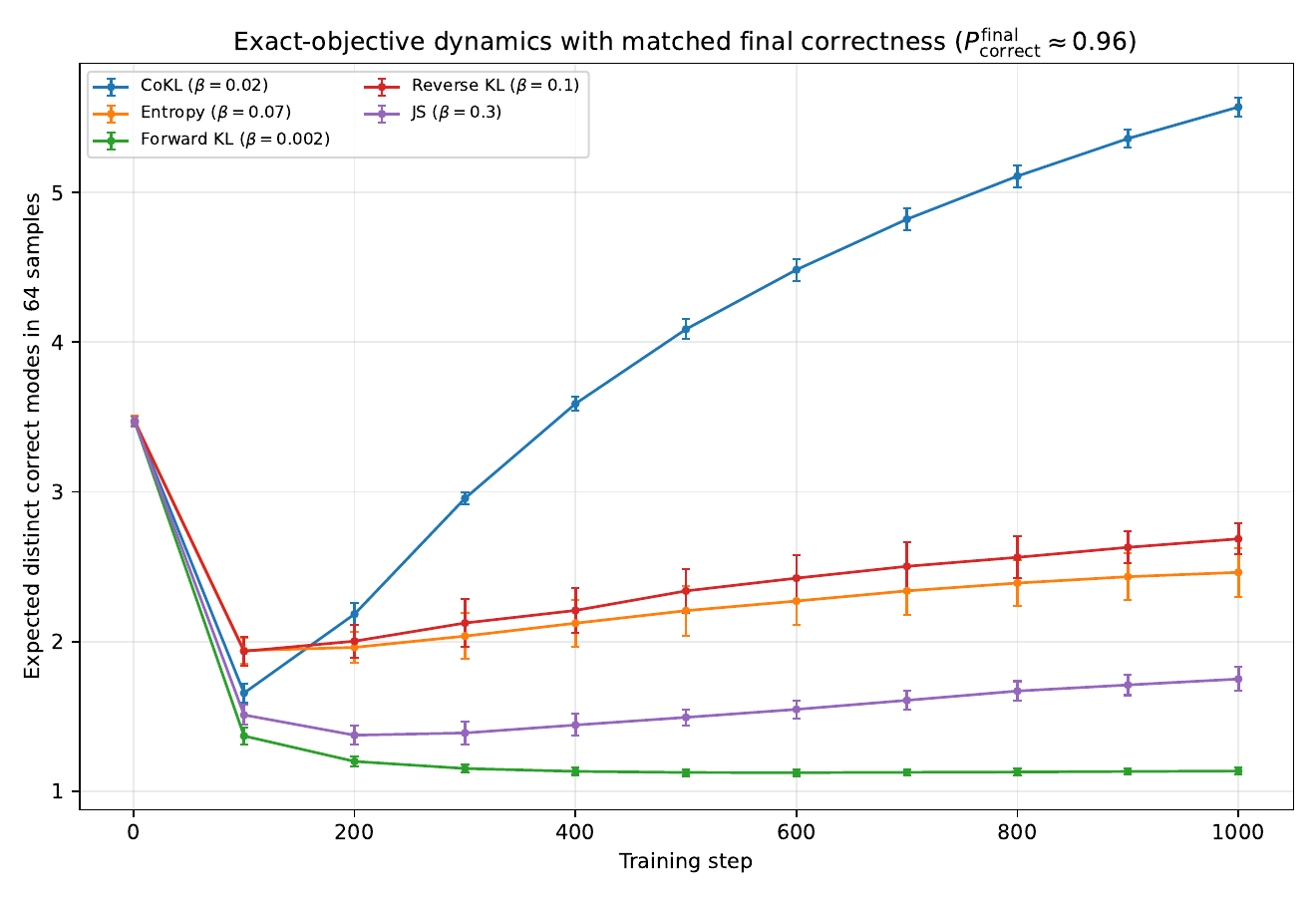}
\end{minipage}
\begin{minipage}[t]{0.23\textwidth}
    \centering
    \includegraphics[width=\linewidth]
    {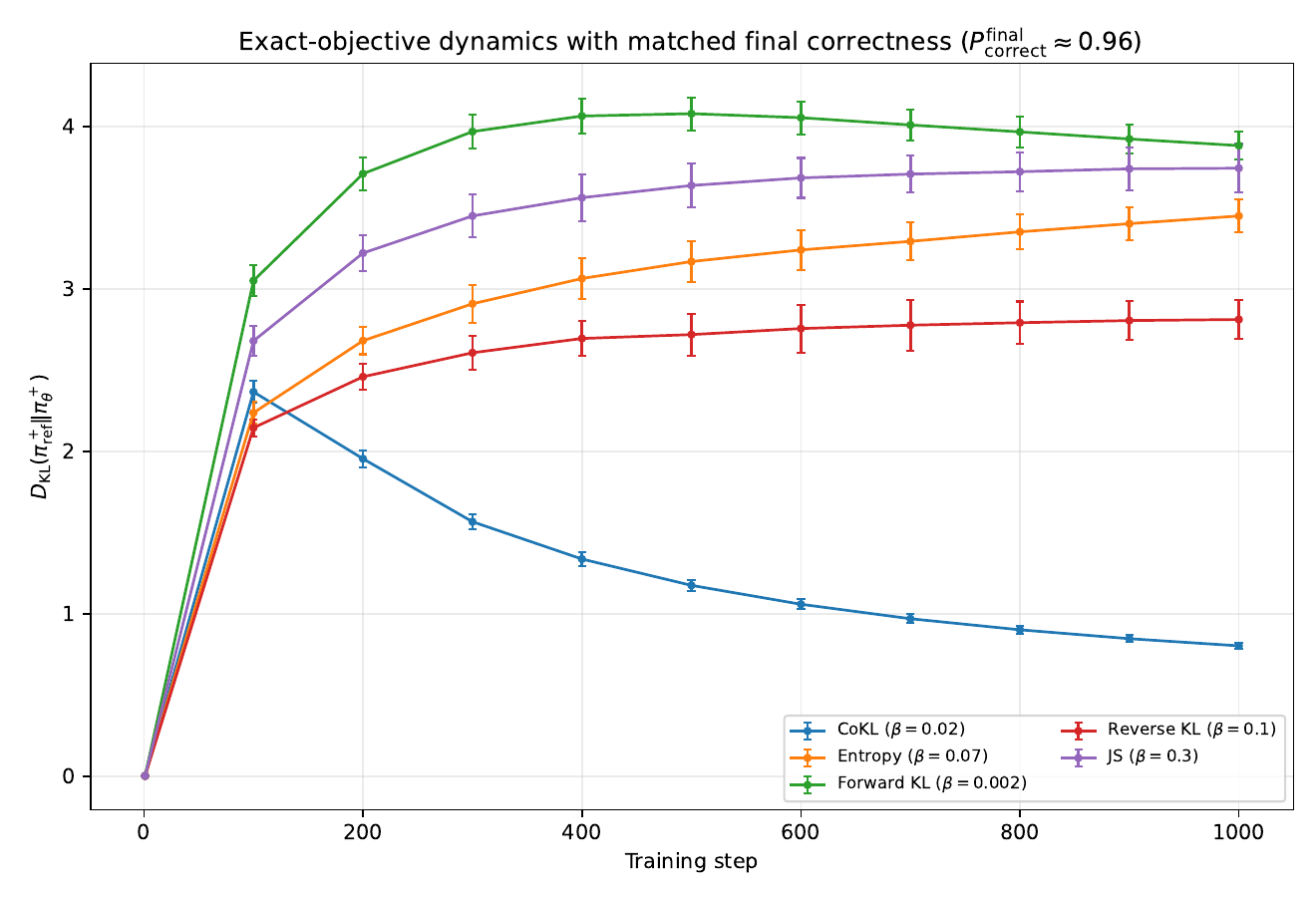}
\end{minipage}

\caption{(Left) Dynamics of correct-mode coverage@64 under the same setting. (Right) Dynamics of
\(D_{\mathrm{KL}}(\pi_{\mathrm{ref}}^+\Vert\pi_\theta^+)\).}
\label{fig:controlled_multi_solution_results_dynamic}
\end{figure}

\begin{table*}[t]
\centering
\small
\renewcommand{\arraystretch}{1.12}
\setlength{\tabcolsep}{4pt}
\begin{tabular*}{\textwidth}
{@{\extracolsep{\fill}}lcccccc}
\toprule
\textbf{Method}
& \multicolumn{4}{c}{\textbf{Normalized Mathematical Reasoning}}
& \multicolumn{1}{c}{\textbf{Normalized Chat}}
& \multicolumn{1}{c}{\textbf{Aggregate}} \\
\cmidrule(lr){2-5}
\cmidrule(lr){6-6}
\cmidrule(lr){7-7}
& \textbf{MATH-Val}
& \textbf{MATH500}
& \textbf{Olympiad}
& \textbf{Math Avg.}
& \textbf{Chat-Val}
& \textbf{Overall} \\
\midrule

\multicolumn{7}{c}{\textit{Qwen3-0.6B}} \\
\addlinespace[2pt]

Base (after Math training)
& 100.00 & 100.00 & 100.00 & 100.00
& 0.00 & 50.00 \\

GRPO w/o KL
& 0.00 & 0.00 & 0.00 & 0.00
& \underline{100.00} & 50.00 \\

Full Reverse-KL
& \textbf{66.79}
& 83.04
& 55.40
& 68.41
& 85.22
& 76.82 \\

Full Forward-KL
& 59.66
& 88.95
& 64.14
& 70.92
& 92.60
& 81.76 \\

Correct-only Reverse-KL
& 46.07
& 52.59
& 45.93
& 48.20
& 87.30
& 67.75 \\

Correct-only Forward-KL
& \underline{64.72}
& \textbf{90.86}
& \textbf{65.08}
& \textbf{73.55}
& 94.65
& \underline{84.10} \\

\textbf{CoKL (ours)}
& 61.96
& \underline{90.35}
& \underline{64.91}
& \underline{72.41}
& \textbf{100.53}
& \textbf{86.47} \\

\midrule

\multicolumn{7}{c}{\textit{Qwen3-1.7B}} \\
\addlinespace[2pt]

Base (after Math training)
& 100.00 & 100.00 & 100.00 & 100.00
& 0.00 & 50.00 \\

GRPO w/o KL
& 0.00 & 0.00 & 0.00 & 0.00
& 100.00 & 50.00 \\

Full Reverse-KL
& 33.17
& 27.55
& 18.37
& 26.36
& \underline{101.81}
& 64.09 \\

Full Forward-KL
& \textbf{88.19}
& 80.77
& 49.74
& 72.90
& 98.14
& 85.52 \\

Correct-only Reverse-KL
& 5.48
& 3.42
& 6.90
& 5.27
& \textbf{102.29}
& 53.78 \\

Correct-only Forward-KL
& 86.68
& \underline{81.73}
& \textbf{51.41}
& \underline{73.28}
& 98.44
& \underline{85.86} \\

\textbf{CoKL (ours)}
& \underline{88.11}
& \textbf{86.63}
& \underline{50.58}
& \textbf{75.11}
& 101.45
& \textbf{88.28} \\

\midrule

\multicolumn{7}{c}{\textit{Qwen3-4B}} \\
\addlinespace[2pt]

Base (after Math training)
& 100.00 & 100.00 & 100.00 & 100.00
& 0.00 & 50.00 \\

GRPO w/o KL
& 0.00 & 0.00 & 0.00 & 0.00
& 100.00 & 50.00 \\

Full Reverse-KL
& -186.44
& -684.64
& -149.25
& -340.11
& 78.25
& -130.93 \\

Full Forward-KL
& \textbf{84.11}
& \underline{39.31}
& \underline{12.37}
& \underline{45.26}
& 101.73
& \underline{73.50} \\

Correct-only Reverse-KL
& 14.32
& -3.66
& 1.18
& 3.95
& 88.97
& 46.46 \\

Correct-only Forward-KL
& \underline{83.88}
& 32.91
& \textbf{12.47}
& 43.09
& \underline{103.65}
& 73.37 \\

\textbf{CoKL (ours)}
& 83.06
& \textbf{43.51}
& 10.75
& \textbf{45.77}
& \textbf{104.02}
& \textbf{74.90} \\

\bottomrule
\end{tabular*}

\caption{Dual-anchor normalized results. Math Avg. averages the three normalized math scores, while Overall equally averages Math Avg. and normalized Chat-Val. Best and second-best post-training results are \textbf{bolded} and \underline{underlined}.}
\label{tab:dual-anchor-normalized-results}
\end{table*}

\subsection{Controlled Multi-Solution RL Environment}
\paragraph{Environment.}
We consider a controlled contextual bandit with multiple valid solutions.
Each input is generated from a latent problem cluster \(z\), for which a
subset \(\mathcal{C}_z\) of the finite action space receives binary reward.
To model multiple reasoning modes, we construct a long-tailed oracle
distribution over the correct actions, pretrain a shared-parameter MLP to
approximate this distribution, and use the resulting frozen model as the
common initialization for all methods. We compare unregularized RL with
entropy, forward KL, reverse KL, JS, and CoKL regularization. All methods are trained using the exact expected-reward objective and exact regularization over the finite action space. Implementation
details are provided in the Appendix E.1.

\paragraph{Results.} For each regularized
method, we sweep the regularization coefficient and evaluate the final
policy using the total correct probability \(P_{\mathrm{corr}}\),
correct-mode coverage@64, and the correctness-conditioned KL
\(D_{\mathrm{KL}}(\pi_{\mathrm{ref}}^+\Vert\pi_\theta^+)\), with all
results averaged over five random seeds. Figure~\ref{fig:controlled_multi_solution_results_tradeoff} shows the resulting
trade-offs. CoKL achieves higher correct-mode coverage and a smaller
correctness-conditioned KL while maintaining high correctness, whereas the
other regularizers either sacrifice correctness or exhibit stronger
deviation from the reference correct-mode distribution. These results
indicate that CoKL better mitigates correct-mode collapse and provides a
more favorable trade-off between correctness improvement and correct-mode
preservation. To further examine the optimization dynamics, we select, for each method,
the coefficient whose final correctness is closest to
\(P_{\mathrm{corr}}=0.96\). Figure~\ref{fig:controlled_multi_solution_results_dynamic}
shows the evolution of coverage@64 and the correctness-conditioned KL
during training. Under comparable final correctness, CoKL consistently
achieves higher correct-mode coverage and a substantially lower
conditional KL, demonstrating a more favorable balance between
correctness improvement and correct-mode preservation.

\subsection{Continual Learning under Task Shift}

\paragraph{Tasks and Datasets.}
We evaluate all methods under a two-stage sequential training protocol consisting of mathematical reasoning followed by general chat, assessing both the retention of previously acquired reasoning capabilities and continued adaptation to the subsequent task. For the math stage, we randomly sample 50K training examples from Nemotron-Math-v2~\cite{du2025nemotronmath} and use an additional 1K examples as the evaluation set. For the chat stage, we use WildChat~\cite{bhaskar2025language} and
Skywork-Reward-Llama-3.1-8B-v0.2~\cite{liu2024skywork} to score model
responses. Dataset statistics are in Appendix~E.2.

\paragraph{Baselines.}
We compare CoKL with the following GRPO-based baselines: (1) GRPO w/o KL removes the KL regularizer; (2) Full Reverse-KL applies reverse-KL regularization to the full response distribution; (3) Full Forward-KL applies forward-KL regularization to the full response distribution using all sampled responses; (4) Correct-only Reverse-KL applies reverse-KL regularization only to verified-correct responses; and (5) Correct-only Forward-KL uses only verified-correct responses for forward-KL-style supervised replay. All baselines use the same GRPO training setup and differ only in the auxiliary regularization objective.

\paragraph{Implementation.} We use the Qwen3 family~\cite{yang2025qwen3} as the base models, including
Qwen3-0.6B, Qwen3-1.7B and Qwen3-4B. For all models, we enable thinking
mode and set the maximum generation length to \(8{,}192\) tokens. Further implementation details are provided in Appendix~E.3, and the
prompt templates in Appendix~H.

\paragraph{Evaluation.}
We evaluate all methods on two task categories. Mathematical reasoning is
evaluated on MATH-500~\cite{hendrycks2021measuring},
OlympiadBench~\cite{he2024olympiadbench}, and the Math-Val validation
set, while chat performance is evaluated on WildChat-Val
~\cite{bhaskar2025language}. We sample with temperature \(0.7\), top-\(p=1.0\), and top-\(k=-1\), and report average@8 throughout. Mathematical reasoning is evaluated by exact match and chat performance by reward-model scores. We apply dual-anchor normalization before aggregation to align their scales; details and unnormalized results are provided in Appendix~E.5.

\paragraph{Main Results.}
As shown in Table~\ref{tab:dual-anchor-normalized-results}, GRPO w/o KL and the Reverse-KL variants often suffer from
substantial forgetting of previously acquired mathematical capabilities.
The Forward-KL variants generally provide stronger capability retention,
but may constrain exploration and adaptation on the new task. CoKL
achieves the most favorable balance between prior-capability retention
and new-task learning, yielding the highest aggregate performance across
all model scales. Compared with the strongest baseline, CoKL improves
Overall by \(2.37\), \(2.42\), and \(1.40\) points on Qwen3-0.6B,
Qwen3-1.7B, and Qwen3-4B, respectively, demonstrating the effectiveness
of the proposed method.

\begin{table}[t]
\centering
\small
\renewcommand{\arraystretch}{1.12}
\setlength{\tabcolsep}{4pt}
\begin{tabular*}{\columnwidth}{@{\extracolsep{\fill}}lccc}
\toprule
\textbf{Method}
& \textbf{Math Avg.}
& \textbf{Chat-Val}
& \textbf{Overall} \\
\midrule
\multicolumn{4}{c}{\textit{Models based on Qwen3-1.7B}} \\
\midrule

CoKL
& \textbf{75.11}
& \textbf{101.45}
& \textbf{88.28} \\

w/o term1
& 3.85
& \underline{100.83}
& 52.34 \\

w/o term2
& \underline{73.28}
& 98.44
& \underline{85.86} \\

\bottomrule
\end{tabular*}

\caption{Ablation results under dual-anchor normalization.}
\label{tab:qwen3-ablation-study}
\end{table}

\subsection{Ablation Studies}
\paragraph{Ablation Study of the Loss Structure.}
To examine the necessity of the two loss terms in CoKL, we conduct an ablation study on Qwen3-1.7B by removing each term in Eq.~17, denoted as \textit{w/o term1} and \textit{w/o term2}, respectively. Removing the second term reduces the objective to the Correct-only Forward-KL formulation. As shown in Table~\ref{tab:qwen3-ablation-study}, w/o term1 leads to a substantial performance degradation, primarily due to pronounced forgetting on the mathematical reasoning tasks. This result confirms that the first term is essential for preserving the model's existing capabilities. In contrast, \textit{w/o term2} underperforms CoKL on the
chat task, demonstrating the importance of the normalization
correction. This term removes the implicit correctness-mass
component of correct-only reference replay, thereby isolating
within-correct distribution preservation and reducing interference
with subsequent-task adaptation.

\paragraph{Batch-Level Correctness Floor.}
Since CoKL does not directly constrain total correctness mass, we examine an explicit batch-level correctness floor. For \(p\in\{\mathrm{ref},\theta\}\), define
\begin{equation}
\widehat{Z}_{p}
=
\frac{1}{NG}\sum_{i=1}^{N}\sum_{j=1}^{G}r^{p}_{i,j},
\qquad
g_{\mathcal B}
=
[\widehat{Z}_{\mathrm{ref}}-\widehat{Z}_{\theta}]_{+}.
\end{equation}
The auxiliary loss and combined objective are
\begin{equation}
\begin{aligned}
\mathcal{L}_{\mathrm{floor}}
&=
-\frac{\operatorname{sg}(g_{\mathcal B})}{NG}
\sum_{i=1}^{N}\sum_{j=1}^{G}
r_{i,j}\ell_{\theta}(y_{i,j}\mid x_i),\\
\mathcal{L}_{\mathrm{total}}
&=
\mathcal{L}_{\mathrm{GRPO}}
+\beta\mathcal{L}_{\mathrm{CoKL}}
+\lambda_{\mathrm{floor}}\mathcal{L}_{\mathrm{floor}}.
\end{aligned}
\label{eq:correctness-floor}
\end{equation}
We denote this variant as \emph{CoKL \(+\mathcal{L}_{\mathrm{floor}}\)}
with \(\lambda_{\mathrm{floor}}=1.0\); since the loss is gated by
\(\operatorname{sg}(g_{\mathcal B})\in[0,1]\), its effective strength
adapts to the correctness deficit rather than acting as a hard constraint.
As shown in Table~\ref{tab:correctness-floor}, this mild, one-sided mass
constraint yields no consistent retention gains while restricting
adaptation and reducing the overall score. This echoes our analysis: an
explicit correctness-mass term reintroduces the coupling between retention
and adaptation that CoKL avoids. Together with the w/o term2 ablation, it
further suggests that, in our setting, forgetting acts more through drift
within the correctness-conditioned distribution than through proportional
loss of correctness mass.

\begin{table}[t]
\centering
\small
\renewcommand{\arraystretch}{1.12}
\setlength{\tabcolsep}{4pt}
\begin{tabular*}{\columnwidth}{@{\extracolsep{\fill}}lccc}
\toprule
\textbf{Method}
& \textbf{Math Avg.}
& \textbf{Chat-Val}
& \textbf{Overall} \\
\midrule

\multicolumn{4}{c}{\textit{Models based on Qwen3-0.6B}} \\
\midrule

CoKL
& \textbf{72.41}
& \textbf{100.53}
& \textbf{86.47} \\

CoKL + $\mathcal{L}_{\mathrm{floor}}$
& 71.65
& 92.49
& 82.07 \\

\midrule

\multicolumn{4}{c}{\textit{Models based on Qwen3-1.7B}} \\
\midrule

CoKL
& 75.11
& \textbf{101.45}
& \textbf{88.28} \\

CoKL + $\mathcal{L}_{\mathrm{floor}}$
& \textbf{75.79}
& 99.63
& 87.71 \\

\midrule

\multicolumn{4}{c}{\textit{Models based on Qwen3-4B}} \\
\midrule

CoKL
& \textbf{45.77}
& \textbf{104.02}
& \textbf{74.90} \\

CoKL + $\mathcal{L}_{\mathrm{floor}}$
& 43.02
& 100.26
& 71.64 \\

\bottomrule
\end{tabular*}

\caption{Ablation of the batch-level correctness floor under dual-anchor normalization. The best result is shown in \textbf{bold}.}
\label{tab:correctness-floor}
\end{table}

\paragraph{KL Batch Analysis.}
We examine KL batch sizes of 8, 16, and 32, as shown in Figure~\ref{fig:cokl_analysis} (Left). Increasing the batch size from 8 to 16 substantially improves mathematical retention, while the additional gain from 16 to 32 is modest. Chat performance varies only slightly across the three settings. These results suggest that larger KL batches provide a more reliable estimate of the regularization objective and can strengthen capability retention without substantially affecting new-task adaptation. To limit the additional rollout cost, we use a KL batch size of 8 in the main experiments.

\paragraph{Sensitivity Analysis of the Loss Coefficient $\beta$.} We analyze $\beta \in \{0.1,1,10\}$. As shown in Figure~\ref{fig:cokl_analysis} (Right), a large $\beta$ restricts chat-task adaptation, whereas a small $\beta$ weakens mathematical capability retention. We therefore set $\beta=1.0$ to balance retention and adaptation.

\begin{figure}[t]
    \centering
    \begin{minipage}[t]{0.48\columnwidth}
        \centering
        \includegraphics[width=\linewidth]{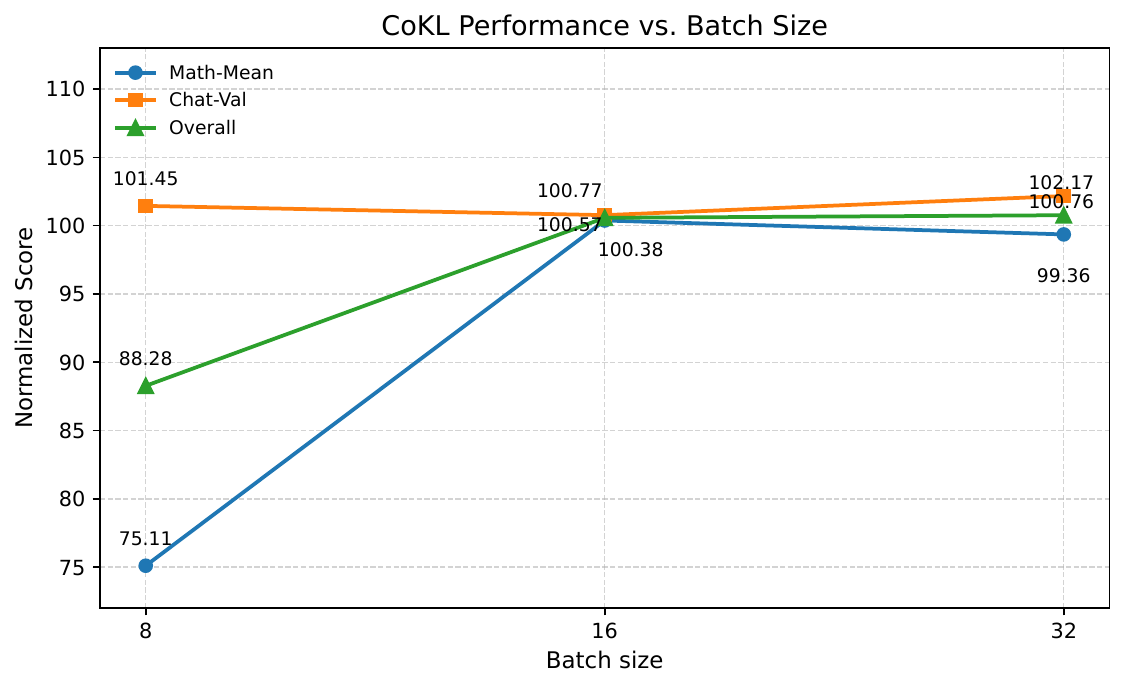}
    \end{minipage}
    \begin{minipage}[t]{0.48\columnwidth}
        \centering
        \includegraphics[width=\linewidth]{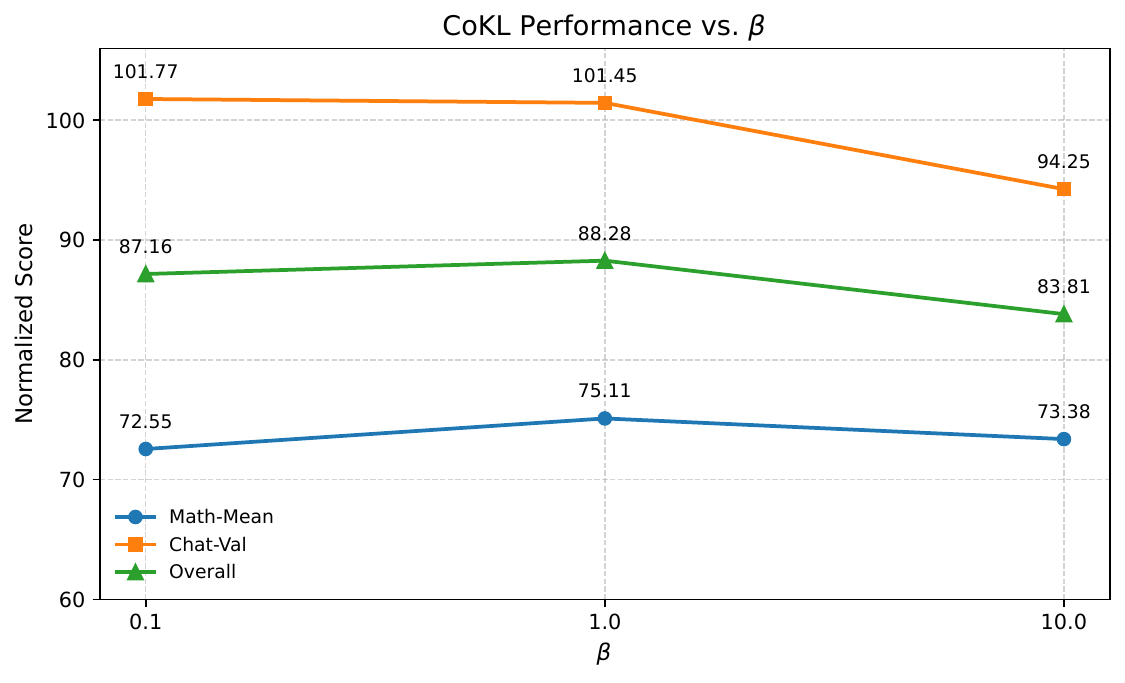}
    \end{minipage}
    \caption{Effects of KL batch size (Left) and loss coefficient $\beta$ (Right) on CoKL performance using Qwen3-1.7B.}
    \label{fig:cokl_analysis}
\end{figure}

\section{Conclusion}
In this paper, we proposed CoKL, a correctness-conditioned
KL regularizer that narrows the preservation constraint from
the full response distribution to the conditional distribution
over verified-correct responses. At the population level under a shared prompt distribution, CoKL
decouples total correctness probability from the relative probability
allocation among correct responses and, unlike
full-policy forward and reverse KL regularization, avoids a
strict optimal correctness gap under an imperfect reference
policy. Controlled multi-solution experiments directly
validate its behavior in preserving the correctness-conditioned
distribution under improved correctness, while continual LLM
post-training experiments demonstrate a favorable practical
trade-off between prior-capability retention and new-task
adaptation.

\bibliography{aaai2027}

\begin{thebibliography}{40}
\providecommand{\natexlab}[1]{#1}

\bibitem[{Bai et~al.(2022)Bai, Jones, Ndousse, Askell, Chen, DasSarma, Drain, Fort, Ganguli, Henighan et~al.}]{bai2022training}
Bai, Y.; Jones, A.; Ndousse, K.; Askell, A.; Chen, A.; DasSarma, N.; Drain, D.; Fort, S.; Ganguli, D.; Henighan, T.; et~al. 2022.
\newblock Training a helpful and harmless assistant with reinforcement learning from human feedback.
\newblock \emph{arXiv preprint arXiv:2204.05862}.

\bibitem[{Bhaskar, Ye, and Chen(2025)}]{bhaskar2025language}
Bhaskar, A.; Ye, X.; and Chen, D. 2025.
\newblock Language models that think, chat better.
\newblock \emph{arXiv preprint arXiv:2509.20357}.

\bibitem[{Cai et~al.(2026)Cai, Liang, Wang, Wang, Zhang, Xia, Sun, Ye, and Shi}]{cai2026advancing}
Cai, M.; Liang, Y.; Wang, L.; Wang, Y.; Zhang, Y.; Xia, L.; Sun, Z.; Ye, X.; and Shi, D. 2026.
\newblock Advancing General-Purpose Reasoning Models with Modular Gradient Surgery.
\newblock \emph{arXiv preprint arXiv:2602.02301}.

\bibitem[{Chen et~al.(2026)Chen, Lu, Zhao, Wang, Yue, Song, and Huang}]{chen2026does}
Chen, Z.; Lu, R.; Zhao, A.; Wang, Z.; Yue, Y.; Song, S.; and Huang, G. 2026.
\newblock Does reinforcement learning really incentivize reasoning capacity in llms beyond the base model?
\newblock \emph{Advances in Neural Information Processing Systems}, 38: 57654--57689.

\bibitem[{Cui et~al.(2025)Cui, Zhang, Chen, Yuan, Wang, Zuo, Li, Fan, Chen, Chen et~al.}]{cui2025entropy}
Cui, G.; Zhang, Y.; Chen, J.; Yuan, L.; Wang, Z.; Zuo, Y.; Li, H.; Fan, Y.; Chen, H.; Chen, W.; et~al. 2025.
\newblock The entropy mechanism of reinforcement learning for reasoning language models.
\newblock \emph{arXiv preprint arXiv:2505.22617}.

\bibitem[{Deng et~al.(2025)Deng, Wei, Yu, and Wu}]{deng2025unlocking}
Deng, W.; Wei, L.; Yu, C.; and Wu, T. 2025.
\newblock Unlocking reasoning capabilities in llms via reinforcement learning exploration.
\newblock \emph{arXiv preprint arXiv:2510.03865}.

\bibitem[{Du et~al.(2025)Du, Toshniwal, Kisacanin, Mahdavi, Moshkov, Armstrong, Ge, Minasyan, Chen, and Gitman}]{du2025nemotronmath}
Du, W.; Toshniwal, S.; Kisacanin, B.; Mahdavi, S.; Moshkov, I.; Armstrong, G.; Ge, S.; Minasyan, E.; Chen, F.; and Gitman, I. 2025.
\newblock Nemotron-Math: Efficient Long-Context Distillation of Mathematical Reasoning from Multi-Mode Supervision.
\newblock \emph{arXiv preprint arXiv:2512.15489}.

\bibitem[{Guo et~al.(2025)Guo, Yang, Zhang, Song, Wang, Zhu, Xu, Zhang, Ma, Bi et~al.}]{guo2025deepseek}
Guo, D.; Yang, D.; Zhang, H.; Song, J.; Wang, P.; Zhu, Q.; Xu, R.; Zhang, R.; Ma, S.; Bi, X.; et~al. 2025.
\newblock Deepseek-r1: Incentivizing reasoning capability in llms via reinforcement learning.
\newblock \emph{arXiv preprint arXiv:2501.12948}.

\bibitem[{GX-Chen et~al.(2025)GX-Chen, Prakash, Guo, Fergus, and Ranganath}]{gx2025kl}
GX-Chen, A.; Prakash, J.; Guo, J.; Fergus, R.; and Ranganath, R. 2025.
\newblock KL-Regularized Reinforcement Learning is Designed to Mode Collapse.
\newblock \emph{arXiv preprint arXiv:2510.20817}.

\bibitem[{He et~al.(2024)He, Luo, Bai, Hu, Thai, Shen, Hu, Han, Huang, Zhang et~al.}]{he2024olympiadbench}
He, C.; Luo, R.; Bai, Y.; Hu, S.; Thai, Z.; Shen, J.; Hu, J.; Han, X.; Huang, Y.; Zhang, Y.; et~al. 2024.
\newblock Olympiadbench: A challenging benchmark for promoting agi with olympiad-level bilingual multimodal scientific problems.
\newblock In \emph{Proceedings of the 62nd Annual Meeting of the Association for Computational Linguistics (Volume 1: Long Papers)}, 3828--3850.

\bibitem[{Hendrycks et~al.(2021)Hendrycks, Burns, Kadavath, Arora, Basart, Tang, Song, and Steinhardt}]{hendrycks2021measuring}
Hendrycks, D.; Burns, C.; Kadavath, S.; Arora, A.; Basart, S.; Tang, E.; Song, D.; and Steinhardt, J. 2021.
\newblock Measuring mathematical problem solving with the math dataset.
\newblock \emph{arXiv preprint arXiv:2103.03874}.

\bibitem[{Kotha, Springer, and Raghunathan(2024)}]{kotha2024understanding}
Kotha, S.; Springer, J.; and Raghunathan, A. 2024.
\newblock Understanding catastrophic forgetting in language models via implicit inference.
\newblock In \emph{International Conference on Learning Representations}, volume 2024, 24110--24139.

\bibitem[{Kruszewski et~al.(2025)Kruszewski, Erbacher, Rozen, and Dymetman}]{kruszewski2025whatever}
Kruszewski, G.; Erbacher, P.; Rozen, J.; and Dymetman, M. 2025.
\newblock Whatever Remains Must Be True: Filtering Drives Reasoning in LLMs, Shaping Diversity.
\newblock \emph{arXiv preprint arXiv:2512.05962}.

\bibitem[{Kwon et~al.(2023)Kwon, Li, Zhuang, Sheng, Zheng, Yu, Gonzalez, Zhang, and Stoica}]{kwon2023efficient}
Kwon, W.; Li, Z.; Zhuang, S.; Sheng, Y.; Zheng, L.; Yu, C.~H.; Gonzalez, J.; Zhang, H.; and Stoica, I. 2023.
\newblock Efficient memory management for large language model serving with pagedattention.
\newblock In \emph{Proceedings of the 29th symposium on operating systems principles}, 611--626.

\bibitem[{Lee, Kang, and Hwang(2026)}]{lee2026sage}
Lee, C.; Kang, M.; and Hwang, S.~J. 2026.
\newblock SAGE: Shaping Anchors for Guided Exploration in RLVR of LLMs.
\newblock \emph{arXiv preprint arXiv:2605.18864}.

\bibitem[{Li et~al.(2025{\natexlab{a}})Li, Zhou, Brunswic, Ghaddar, Sun, Ma, Luo, Li, Coates, Hao et~al.}]{li2025omni}
Li, D.; Zhou, J.; Brunswic, L.~M.; Ghaddar, A.; Sun, Q.; Ma, L.; Luo, Y.; Li, D.; Coates, M.; Hao, J.; et~al. 2025{\natexlab{a}}.
\newblock Omni-Thinker: Scaling Multi-Task RL in LLMs with Hybrid Reward and Task Scheduling.
\newblock \emph{arXiv preprint arXiv:2507.14783}.

\bibitem[{Li et~al.(2025{\natexlab{b}})Li, Zhou, Hao, Liu, Miao, Pang, Tan, Chu, Wang, Pan et~al.}]{li2025choice}
Li, L.; Zhou, Z.; Hao, J.; Liu, J.~K.; Miao, Y.; Pang, W.; Tan, X.; Chu, W.; Wang, Z.; Pan, S.; et~al. 2025{\natexlab{b}}.
\newblock The choice of divergence: A neglected key to mitigating diversity collapse in reinforcement learning with verifiable reward.
\newblock \emph{arXiv preprint arXiv:2509.07430}.

\bibitem[{Li et~al.(2025{\natexlab{c}})Li, Pan, Lin, Sun, He, and Wu}]{li2025can}
Li, Y.; Pan, Z.; Lin, H.; Sun, M.; He, C.; and Wu, L. 2025{\natexlab{c}}.
\newblock Can one domain help others? a data-centric study on multi-domain reasoning via reinforcement learning.
\newblock \emph{arXiv preprint arXiv:2507.17512}.

\bibitem[{Lin et~al.(2026{\natexlab{a}})Lin, Gong, Tang, Li, Wang, Ma, Huang, Tang, and Lu}]{lin2026expo}
Lin, M.; Gong, Z.; Tang, M.; Li, Q.; Wang, C.; Ma, J.; Huang, S.; Tang, K.; and Lu, H. 2026{\natexlab{a}}.
\newblock expo: Exploration-prioritized policy optimization via adaptive kl regulation and gaussian curriculum sampling.
\newblock \emph{arXiv preprint arXiv:2605.09923}.

\bibitem[{Lin et~al.(2024)Lin, Lin, Xiong, Diao, Liu, Zhang, Pan, Wang, Hu, Zhang et~al.}]{lin2024mitigating}
Lin, Y.; Lin, H.; Xiong, W.; Diao, S.; Liu, J.; Zhang, J.; Pan, R.; Wang, H.; Hu, W.; Zhang, H.; et~al. 2024.
\newblock Mitigating the alignment tax of rlhf.
\newblock In \emph{Proceedings of the 2024 Conference on Empirical Methods in Natural Language Processing}, 580--606.

\bibitem[{Lin et~al.(2026{\natexlab{b}})Lin, Wang, Cao, Chai, Wang, Lu, Lin, He, and Yin}]{lin2026resrl}
Lin, Z.; Wang, X.; Cao, J.; Chai, J.; Wang, L.; Lu, X.; Lin, W.; He, R.; and Yin, G. 2026{\natexlab{b}}.
\newblock Resrl: Boosting llm reasoning via negative sample projection residual reinforcement learning.
\newblock \emph{arXiv preprint arXiv:2605.00380}.

\bibitem[{Lin et~al.(2025)Lin, Wang, Cao, Chai, Yin, Lin, and He}]{lin2025rest}
Lin, Z.; Wang, X.; Cao, J.; Chai, J.; Yin, G.; Lin, W.; and He, R. 2025.
\newblock ResT: Reshaping Token-Level Policy Gradients for Tool-Use Large Language Models.
\newblock \emph{arXiv preprint arXiv:2509.21826}.

\bibitem[{Liu et~al.(2024)Liu, Zeng, Liu, Yan, He, Wang, Yan, Liu, and Zhou}]{liu2024skywork}
Liu, C.~Y.; Zeng, L.; Liu, J.; Yan, R.; He, J.; Wang, C.; Yan, S.; Liu, Y.; and Zhou, Y. 2024.
\newblock Skywork-Reward: Bag of Tricks for Reward Modeling in LLMs.
\newblock \emph{arXiv preprint arXiv:2410.18451}.

\bibitem[{Liu et~al.(2025)Liu, Chen, Li, Qi, Pang, Du, Lee, and Lin}]{liu2025understanding}
Liu, Z.; Chen, C.; Li, W.; Qi, P.; Pang, T.; Du, C.; Lee, W.~S.; and Lin, M. 2025.
\newblock Understanding r1-zero-like training: A critical perspective.
\newblock \emph{arXiv preprint arXiv:2503.20783}.

\bibitem[{Lochab, Li, and Zhang(2026)}]{lochab2026uniform}
Lochab, A.; Li, B.; and Zhang, R. 2026.
\newblock Uniform-Correct Policy Optimization: Breaking RLVR's Indifference to Diversity.
\newblock \emph{arXiv preprint arXiv:2605.00365}.

\bibitem[{Lu et~al.(2026)Lu, Wang, Chai, Yin, Lin, Chen, Luo, Zhuang, Ban, and Wang}]{lu2026contextual}
Lu, X.; Wang, X.; Chai, J.; Yin, G.; Lin, W.; Chen, Z.; Luo, Y.; Zhuang, F.; Ban, Y.; and Wang, D. 2026.
\newblock Contextual Rollout Bandits for Reinforcement Learning with Verifiable Rewards.
\newblock \emph{arXiv preprint arXiv:2602.08499}.

\bibitem[{Ouyang et~al.(2022)Ouyang, Wu, Jiang, Almeida, Wainwright, Mishkin, Zhang, Agarwal, Slama, Ray et~al.}]{ouyang2022training}
Ouyang, L.; Wu, J.; Jiang, X.; Almeida, D.; Wainwright, C.; Mishkin, P.; Zhang, C.; Agarwal, S.; Slama, K.; Ray, A.; et~al. 2022.
\newblock Training language models to follow instructions with human feedback.
\newblock \emph{Advances in neural information processing systems}, 35: 27730--27744.

\bibitem[{Shao et~al.(2024)Shao, Wang, Zhu, Xu, Song, Bi, Zhang, Zhang, Li, Wu et~al.}]{shao2024deepseekmath}
Shao, Z.; Wang, P.; Zhu, Q.; Xu, R.; Song, J.; Bi, X.; Zhang, H.; Zhang, M.; Li, Y.; Wu, Y.; et~al. 2024.
\newblock Deepseekmath: Pushing the limits of mathematical reasoning in open language models.
\newblock \emph{arXiv preprint arXiv:2402.03300}.

\bibitem[{Sheng et~al.(2024)Sheng, Zhang, Ye, Wu, Zhang, Zhang, Peng, Lin, and Wu}]{sheng2024hybridflow}
Sheng, G.; Zhang, C.; Ye, Z.; Wu, X.; Zhang, W.; Zhang, R.; Peng, Y.; Lin, H.; and Wu, C. 2024.
\newblock HybridFlow: A Flexible and Efficient RLHF Framework.
\newblock \emph{arXiv preprint arXiv: 2409.19256}.

\bibitem[{Stiennon et~al.(2020)Stiennon, Ouyang, Wu, Ziegler, Lowe, Voss, Radford, Amodei, and Christiano}]{stiennon2020learning}
Stiennon, N.; Ouyang, L.; Wu, J.; Ziegler, D.; Lowe, R.; Voss, C.; Radford, A.; Amodei, D.; and Christiano, P.~F. 2020.
\newblock Learning to summarize with human feedback.
\newblock \emph{Advances in neural information processing systems}, 33: 3008--3021.

\bibitem[{Vassoyan, Beau, and Plaud(2025)}]{vassoyan2025ignore}
Vassoyan, J.; Beau, N.; and Plaud, R. 2025.
\newblock Ignore the kl penalty! boosting exploration on critical tokens to enhance rl fine-tuning.
\newblock In \emph{Findings of the Association for Computational Linguistics: NAACL 2025}, 6108--6118.

\bibitem[{Wan et~al.(2026)Wan, Shen, Dou, Zhou, Zhang, Wang, Shen, Xiong, Tao, Zhong et~al.}]{wan2026dsdr}
Wan, Z.; Shen, Y.; Dou, Z.; Zhou, D.; Zhang, Y.; Wang, X.; Shen, H.; Xiong, J.; Tao, C.; Zhong, Z.; et~al. 2026.
\newblock Dsdr: Dual-scale diversity regularization for exploration in llm reasoning.
\newblock \emph{arXiv preprint arXiv:2602.19895}.

\bibitem[{Wang et~al.(2026{\natexlab{a}})Wang, Long, Li, Xu, Li, and Tang}]{wang2026mix}
Wang, H.; Long, X.; Li, Z.; Xu, Y.; Li, T.; and Tang, Y. 2026{\natexlab{a}}.
\newblock To Mix or To Merge: Toward Multi-Domain Reinforcement Learning for Large Language Models.
\newblock \emph{arXiv preprint arXiv:2602.12566}.

\bibitem[{Wang et~al.(2025)Wang, Liu, Zhang, Li, Zhou, and Pan}]{wang2025stabilizing}
Wang, J.; Liu, R.; Zhang, F.; Li, X.; Zhou, G.; and Pan, L. 2025.
\newblock Stabilizing knowledge, promoting reasoning: Dual-token constraints for rlvr.
\newblock \emph{arXiv preprint arXiv:2507.15778}.

\bibitem[{Wang et~al.(2026{\natexlab{b}})Wang, Wang, Lu, Zhang, Wu, Chai, Lin, and Yin}]{wang2026implicit}
Wang, L.; Wang, X.; Lu, X.; Zhang, Z.; Wu, J.; Chai, J.; Lin, W.; and Yin, G. 2026{\natexlab{b}}.
\newblock Implicit Hierarchical GRPO: Decoupling Tool Invocation from Execution for Tool-Integrated Mathematical Reasoning.
\newblock \emph{arXiv preprint arXiv:2605.18500}.

\bibitem[{Xue et~al.(2025)Xue, Zheng, Liu, Li, Zheng, Ma, and An}]{xue2025simpletir}
Xue, Z.; Zheng, L.; Liu, Q.; Li, Y.; Zheng, X.; Ma, Z.; and An, B. 2025.
\newblock Simpletir: End-to-end reinforcement learning for multi-turn tool-integrated reasoning.
\newblock \emph{arXiv preprint arXiv:2509.02479}.

\bibitem[{Yang et~al.(2025)Yang, Li, Yang, Zhang, Hui, Zheng, Yu, Gao, Huang, Lv et~al.}]{yang2025qwen3}
Yang, A.; Li, A.; Yang, B.; Zhang, B.; Hui, B.; Zheng, B.; Yu, B.; Gao, C.; Huang, C.; Lv, C.; et~al. 2025.
\newblock Qwen3 technical report.
\newblock \emph{arXiv preprint arXiv:2505.09388}.

\bibitem[{Yu et~al.(2026)Yu, Zhang, Zhu, Yuan, Zuo, Yue, Dai, Fan, Liu, Liu et~al.}]{yu2026dapo}
Yu, Q.; Zhang, Z.; Zhu, R.; Yuan, Y.; Zuo, X.; Yue, Y.; Dai, W.; Fan, T.; Liu, G.; Liu, L.; et~al. 2026.
\newblock Dapo: An open-source llm reinforcement learning system at scale.
\newblock \emph{Advances in Neural Information Processing Systems}, 38: 113222--113244.

\bibitem[{Zheng et~al.(2025)Zheng, Liu, Li, Chen, Yu, Gao, Dang, Liu, Men, Yang et~al.}]{zheng2025group}
Zheng, C.; Liu, S.; Li, M.; Chen, X.-H.; Yu, B.; Gao, C.; Dang, K.; Liu, Y.; Men, R.; Yang, A.; et~al. 2025.
\newblock Group sequence policy optimization.
\newblock \emph{arXiv preprint arXiv:2507.18071}.

\bibitem[{Ziegler et~al.(2019)Ziegler, Stiennon, Wu, Brown, Radford, Amodei, Christiano, and Irving}]{ziegler2019fine}
Ziegler, D.~M.; Stiennon, N.; Wu, J.; Brown, T.~B.; Radford, A.; Amodei, D.; Christiano, P.; and Irving, G. 2019.
\newblock Fine-tuning language models from human preferences.
\newblock \emph{arXiv preprint arXiv:1909.08593}.

\end{thebibliography}

\appendix
\section{Appendix A: Derivation of the Correctness Decomposition}
\label{app:kl-correctness-decomposition}
For a fixed prompt \(x\), let \(r(x,y)\in\{0,1\}\) indicate whether response \(y\) is verified as correct. For continuous rewards, samples can be partitioned into positive and negative groups using a threshold, allowing the same analysis to apply. Define the correct and incorrect
response sets as
\[
\mathcal{Y}^{+}(x)=\{y:r(x,y)=1\},
\qquad
\mathcal{Y}^{-}(x)=\{y:r(x,y)=0\}.
\]
For any policy \(\pi\), its probability of producing a correct response is
\begin{equation}
Z_{\pi}(x)
=
\sum_y r(x,y)\pi(y\mid x).
\label{eq:app-correctness-probability}
\end{equation}
The corresponding correctness-conditioned and
incorrectness-conditioned policies are
\begin{equation}
\begin{aligned}
\pi^{+}(y\mid x)
&=
\frac{r(x,y)\pi(y\mid x)}
{Z_{\pi}(x)},\\
\pi^{-}(y\mid x)
&=
\frac{[1-r(x,y)]\pi(y\mid x)}
{1-Z_{\pi}(x)}.
\end{aligned}
\label{eq:app-conditioned-policies}
\end{equation}

Let \(\pi_{\mathrm{ref}}\) denote the reference policy and
\(\pi_{\theta}\) the current policy. For readability, we omit the prompt
\(x\) below and write
\[
Z_{\mathrm{ref}}
=
Z_{\pi_{\mathrm{ref}}}(x),
\qquad
Z_{\theta}
=
Z_{\pi_{\theta}}(x).
\]
We assume \(0<Z_{\mathrm{ref}},Z_{\theta}<1\) for notational
convenience. The same decomposition extends to boundary cases under the
standard extended-value convention for KL divergence. The full-policy forward KL can be partitioned over the correct and
incorrect response sets:
\begin{equation}
\begin{aligned}
D_{\mathrm{KL}}
\left(
\pi_{\mathrm{ref}}
\|
\pi_{\theta}
\right)
={}&
\sum_{y\in\mathcal{Y}^{+}}
\pi_{\mathrm{ref}}(y)
\log
\frac{\pi_{\mathrm{ref}}(y)}
{\pi_{\theta}(y)}
\\
&+
\sum_{y\in\mathcal{Y}^{-}}
\pi_{\mathrm{ref}}(y)
\log
\frac{\pi_{\mathrm{ref}}(y)}
{\pi_{\theta}(y)}.
\end{aligned}
\label{eq:app-kl-partition}
\end{equation}

For \(y\in\mathcal{Y}^{+}\), we have
\[
\pi_{\mathrm{ref}}(y)
=
Z_{\mathrm{ref}}\pi_{\mathrm{ref}}^{+}(y),
\qquad
\pi_{\theta}(y)
=
Z_{\theta}\pi_{\theta}^{+}(y).
\]
Therefore, the contribution from the correct-response set satisfies
\begin{equation}
\begin{aligned}
&
\sum_{y\in\mathcal{Y}^{+}}
\pi_{\mathrm{ref}}(y)
\log
\frac{\pi_{\mathrm{ref}}(y)}
{\pi_{\theta}(y)}
\\
={}&
Z_{\mathrm{ref}}
\sum_{y\in\mathcal{Y}^{+}}
\pi_{\mathrm{ref}}^{+}(y)
\log
\frac{
Z_{\mathrm{ref}}\pi_{\mathrm{ref}}^{+}(y)
}{
Z_{\theta}\pi_{\theta}^{+}(y)
}
\\
={}&
Z_{\mathrm{ref}}
\log
\frac{Z_{\mathrm{ref}}}{Z_{\theta}}
+
Z_{\mathrm{ref}}
D_{\mathrm{KL}}
\left(
\pi_{\mathrm{ref}}^{+}
\|
\pi_{\theta}^{+}
\right).
\end{aligned}
\label{eq:app-positive-part}
\end{equation}
Here, the last equality uses
\[
\sum_{y\in\mathcal{Y}^{+}}
\pi_{\mathrm{ref}}^{+}(y)=1.
\]

Similarly, for \(y\in\mathcal{Y}^{-}\),
\[
\pi_{\mathrm{ref}}(y)
=
(1-Z_{\mathrm{ref}})
\pi_{\mathrm{ref}}^{-}(y),
\qquad
\pi_{\theta}(y)
=
(1-Z_{\theta})
\pi_{\theta}^{-}(y).
\]
The contribution from the incorrect-response set is
\begin{equation}
\begin{aligned}
&
\sum_{y\in\mathcal{Y}^{-}}
\pi_{\mathrm{ref}}(y)
\log
\frac{\pi_{\mathrm{ref}}(y)}
{\pi_{\theta}(y)}
\\
={}&
(1-Z_{\mathrm{ref}})
\log
\frac{1-Z_{\mathrm{ref}}}
{1-Z_{\theta}}
\\
&+
(1-Z_{\mathrm{ref}})
D_{\mathrm{KL}}
\left(
\pi_{\mathrm{ref}}^{-}
\|
\pi_{\theta}^{-}
\right).
\end{aligned}
\label{eq:app-negative-part}
\end{equation}

Combining Eqs.~\ref{eq:app-positive-part}
and~\ref{eq:app-negative-part} gives
\begin{equation}
\begin{aligned}
D_{\mathrm{KL}}
\left(
\pi_{\mathrm{ref}}
\|
\pi_{\theta}
\right)
={}&
Z_{\mathrm{ref}}
\log
\frac{Z_{\mathrm{ref}}}{Z_{\theta}}
+
(1-Z_{\mathrm{ref}})
\log
\frac{1-Z_{\mathrm{ref}}}
{1-Z_{\theta}}
\\
&+
Z_{\mathrm{ref}}
D_{\mathrm{KL}}
\left(
\pi_{\mathrm{ref}}^{+}
\|
\pi_{\theta}^{+}
\right)
\\
&+
(1-Z_{\mathrm{ref}})
D_{\mathrm{KL}}
\left(
\pi_{\mathrm{ref}}^{-}
\|
\pi_{\theta}^{-}
\right).
\end{aligned}
\label{eq:app-expanded-decomposition}
\end{equation}

The first two terms form the KL divergence between the Bernoulli
distributions induced by the correctness event:
\begin{equation}
\begin{aligned}
&
D_{\mathrm{KL}}
\left(
\operatorname{Bern}(Z_{\mathrm{ref}})
\|
\operatorname{Bern}(Z_{\theta})
\right)
\\
={}&
Z_{\mathrm{ref}}
\log
\frac{Z_{\mathrm{ref}}}{Z_{\theta}}
+
(1-Z_{\mathrm{ref}})
\log
\frac{1-Z_{\mathrm{ref}}}
{1-Z_{\theta}}.
\end{aligned}
\label{eq:app-bernoulli-kl}
\end{equation}
Hence, the full-policy forward KL admits the decomposition
\begin{equation}
\begin{aligned}
D_{\mathrm{KL}}
\left(
\pi_{\mathrm{ref}}
\|
\pi_{\theta}
\right)
={}&
D_{\mathrm{KL}}
\left(
\operatorname{Bern}(Z_{\mathrm{ref}})
\|
\operatorname{Bern}(Z_{\theta})
\right)
\\
&+
Z_{\mathrm{ref}}
D_{\mathrm{KL}}
\left(
\pi_{\mathrm{ref}}^{+}
\|
\pi_{\theta}^{+}
\right)
\\
&+
(1-Z_{\mathrm{ref}})
D_{\mathrm{KL}}
\left(
\pi_{\mathrm{ref}}^{-}
\|
\pi_{\theta}^{-}
\right).
\end{aligned}
\label{eq:app-final-decomposition}
\end{equation}

\section{Appendix B: Complete Derivation of CoKL}
\label{app:cokl-full-derivation}

\subsection{Setup}

For each prompt \(x\), let
\(\pi_{\mathrm{ref}}(y\mid x)\) denote the frozen reference policy and
\(\pi_\theta(y\mid x)\) denote the current policy. Let \(Y=1\) denote
the event that response \(y\) is verified as correct for prompt \(x\). We define $
r(x,y)
=
\mathbf{1}
\bigl[
y\text{ is correct}
\bigr]\in\{0,1\}$. The correctness probabilities of the reference and current policies are
\begin{equation}
Z_{\mathrm{ref}}(x)
=
\sum_y
r(x,y)\pi_{\mathrm{ref}}(y\mid x),
\label{eq:appb-z-ref}
\end{equation}
and
\begin{equation}
Z_\theta(x)
=
\sum_y
r(x,y)\pi_\theta(y\mid x).
\label{eq:appb-z-theta}
\end{equation}

The corresponding correctness-conditioned policies are
\begin{equation}
\pi_{\mathrm{ref}}^+(y\mid x)
=
\frac{
r(x,y)\pi_{\mathrm{ref}}(y\mid x)
}{
Z_{\mathrm{ref}}(x)
},
\label{eq:appb-ref-positive}
\end{equation}
and
\begin{equation}
\pi_\theta^+(y\mid x)
=
\frac{
r(x,y)\pi_\theta(y\mid x)
}{
Z_\theta(x)
}.
\label{eq:appb-theta-positive}
\end{equation}
We assume that \(Z_{\mathrm{ref}}(x)>0\) and \(Z_\theta(x)>0\), so that
the correctness-conditioned policies are well-defined.

\subsection{CoKL Objective and Equivalent Form}

CoKL is defined as the forward KL divergence between the
correctness-conditioned reference and current policies:
\begin{equation}
\mathcal{L}_{\mathrm{CoKL}}(x)
=
D_{\mathrm{KL}}
\left(
\pi_{\mathrm{ref}}^+(\cdot\mid x)
\|
\pi_\theta^+(\cdot\mid x)
\right).
\label{eq:appb-cokl-definition}
\end{equation}
Expanding the KL divergence gives
\begin{equation}
\begin{aligned}
\mathcal{L}_{\mathrm{CoKL}}(x)
={}&
\sum_y
\pi_{\mathrm{ref}}^+(y\mid x)
\log \pi_{\mathrm{ref}}^+(y\mid x)
\\
&-
\sum_y
\pi_{\mathrm{ref}}^+(y\mid x)
\log \pi_\theta^+(y\mid x).
\end{aligned}
\label{eq:appb-cokl-expansion}
\end{equation}

The first term in
Eq.~\ref{eq:appb-cokl-expansion} is independent of \(\theta\). Moreover,
every response in the support of
\(\pi_{\mathrm{ref}}^+(\cdot\mid x)\) satisfies \(r(x,y)=1\). Therefore,
using Eq.~\ref{eq:appb-theta-positive},
\begin{equation}
\log\pi_\theta^+(y\mid x)
=
\log\pi_\theta(y\mid x)
-
\log Z_\theta(x)
\label{eq:appb-conditioned-logprob}
\end{equation}
on the support of \(\pi_{\mathrm{ref}}^+(\cdot\mid x)\). Substituting Eq.~\ref{eq:appb-conditioned-logprob} into
Eq.~\ref{eq:appb-cokl-expansion} yields, up to terms independent of
\(\theta\),
\begin{equation}
\begin{aligned}
\mathcal{L}_{\mathrm{CoKL}}(x)
\doteq{}&
-
\mathbb{E}_{
y\sim\pi_{\mathrm{ref}}^+(\cdot\mid x)
}
\left[
\log\pi_\theta(y\mid x)
\right]
+
\log Z_\theta(x),
\end{aligned}
\label{eq:appb-cokl-equivalent}
\end{equation}
where \(\doteq\) denotes equality up to terms independent of
\(\theta\).

\subsection{Gradient of the CoKL Objective}

Differentiating Eq.~\ref{eq:appb-cokl-equivalent}, the first term gives
\begin{equation}
-
\mathbb{E}_{
y\sim\pi_{\mathrm{ref}}^+(\cdot\mid x)
}
\left[
\nabla_\theta\log\pi_\theta(y\mid x)
\right].
\label{eq:appb-reference-gradient}
\end{equation}

For the normalization term,
\begin{equation}
\nabla_\theta\log Z_\theta(x)
=
\frac{
\nabla_\theta Z_\theta(x)
}{
Z_\theta(x)
}.
\label{eq:appb-logz-gradient}
\end{equation}
From Eq.~\ref{eq:appb-z-theta},
\begin{equation}
\nabla_\theta Z_\theta(x)
=
\sum_y
r(x,y)
\nabla_\theta\pi_\theta(y\mid x).
\label{eq:appb-z-gradient}
\end{equation}
Using
\[
\nabla_\theta\pi_\theta(y\mid x)
=
\pi_\theta(y\mid x)
\nabla_\theta\log\pi_\theta(y\mid x),
\]
we obtain
\begin{equation}
\begin{aligned}
\nabla_\theta\log Z_\theta(x)
={}&
\sum_y
\frac{
r(x,y)\pi_\theta(y\mid x)
}{
Z_\theta(x)
}
\\
&\qquad\cdot
\nabla_\theta\log\pi_\theta(y\mid x).
\end{aligned}
\label{eq:appb-logz-expanded}
\end{equation}
By Eq.~\ref{eq:appb-theta-positive}, this becomes
\begin{equation}
\nabla_\theta\log Z_\theta(x)
=
\mathbb{E}_{
y\sim\pi_\theta^+(\cdot\mid x)
}
\left[
\nabla_\theta\log\pi_\theta(y\mid x)
\right].
\label{eq:appb-logz-expectation}
\end{equation}

Combining Eqs.~\ref{eq:appb-reference-gradient}
and~\ref{eq:appb-logz-expectation}, the CoKL gradient is
\begin{equation}
\begin{aligned}
\nabla_\theta
\mathcal{L}_{\mathrm{CoKL}}(x)
={}&
-
\mathbb{E}_{
y\sim\pi_{\mathrm{ref}}^+(\cdot\mid x)
}
\left[
\nabla_\theta\log\pi_\theta(y\mid x)
\right]
\\
&+
\mathbb{E}_{
y\sim\pi_\theta^+(\cdot\mid x)
}
\left[
\nabla_\theta\log\pi_\theta(y\mid x)
\right].
\end{aligned}
\label{eq:appb-cokl-gradient}
\end{equation}

\subsection{On-Policy Monte Carlo Surrogate}

Consider a minibatch containing \(N\) prompt groups
\(\{x_i\}_{i=1}^N\). For each prompt \(x_i\), draw \(G\) responses from
the frozen reference policy:
\begin{equation}
y_{ij}
\sim
\pi_{\mathrm{ref}}(\cdot\mid x_i),
\qquad
r_{ij}
=
r(x_i,y_{ij}).
\label{eq:appb-reference-samples}
\end{equation}
The empirical reference-side conditional weights are
\begin{equation}
\alpha_{ij}
=
\frac{
r_{ij}
}{
\sum_{k=1}^G r_{ik}
}.
\label{eq:appb-alpha}
\end{equation}
When \(\sum_{k=1}^G r_{ik}=0\), we set
\(\alpha_{ij}=0\) for all responses in that group. For any function \(f\), the reference-side conditional expectation can
be approximated by the self-normalized estimator
\begin{equation}
\mathbb{E}_{
y\sim\pi_{\mathrm{ref}}^+(\cdot\mid x_i)
}
[f(y)]
\approx
\sum_{j=1}^G
\alpha_{ij}f(y_{ij}).
\label{eq:appb-reference-mc}
\end{equation}

Similarly, draw \(G\) fresh responses from the current policy:
\begin{equation}
y'_{ij}
\sim
\pi_\theta(\cdot\mid x_i),
\qquad
r'_{ij}
=
r(x_i,y'_{ij}).
\label{eq:appb-current-samples}
\end{equation}
Define the empirical current-policy conditional weights as
\begin{equation}
\delta_{ij}
=
\frac{
r'_{ij}
}{
\sum_{k=1}^G r'_{ik}
}.
\label{eq:appb-delta}
\end{equation}
When \(\sum_{k=1}^G r'_{ik}=0\), we set
\(\delta_{ij}=0\) for all responses in that group. The current-policy conditional expectation is then approximated by
\begin{equation}
\mathbb{E}_{
y\sim\pi_\theta^+(\cdot\mid x_i)
}
[f(y)]
\approx
\sum_{j=1}^G
\delta_{ij}f(y'_{ij}).
\label{eq:appb-current-mc}
\end{equation}

A stop-gradient surrogate whose gradient implements the finite-group
approximation of Eq.~\ref{eq:appb-cokl-gradient} is
\begin{equation}
\begin{aligned}
\widehat{\mathcal{L}}_{\mathrm{on}}
={}&
-
\frac{1}{N}
\sum_{i=1}^N
\sum_{j=1}^G
\operatorname{sg}(\alpha_{ij})
\log\pi_\theta(y_{ij}\mid x_i)
\\
&+
\frac{1}{N}
\sum_{i=1}^N
\sum_{j=1}^G
\operatorname{sg}(\delta_{ij})
\log\pi_\theta(y'_{ij}\mid x_i),
\end{aligned}
\label{eq:appb-on-policy-loss}
\end{equation}
where \(\operatorname{sg}(\cdot)\) denotes stop-gradient. The second
term is a gradient surrogate for \(\log Z_\theta(x)\), rather than an
unbiased scalar estimator of \(\log Z_\theta(x)\).

The zero-denominator convention above makes the
current-policy conditional term vanish when no
verified-correct response is observed in a rollout group.
We next characterize the resulting expected finite-group
update in the on-policy setting.
\paragraph{Finite-Group Correctness-Recovery Property.}
For a fixed prompt \(x\), let
\[
s_\theta(y\mid x)
=
\nabla_\theta \log \pi_\theta(y\mid x),
\]
and define
\[
\mu_{\mathrm{ref}}^{+}(x)
=
\mathbb{E}_{y\sim\pi_{\mathrm{ref}}^{+}(\cdot\mid x)}
\left[
s_\theta(y\mid x)
\right],
\]
\[
\mu_{\theta}^{+}(x)
=
\mathbb{E}_{y\sim\pi_{\theta}^{+}(\cdot\mid x)}
\left[
s_\theta(y\mid x)
\right]
=
\nabla_\theta\log Z_\theta(x).
\]
Consider the on-policy finite-group surrogate in
Eq.~\ref{eq:appb-on-policy-loss}, where the reference group is conditioned on
containing at least one verified-correct response, as in the
construction of \(\mathcal{D}^{+}\). Let
\[
K_\theta
=
\sum_{j=1}^{G} r(x,y'_j),
\qquad
y'_j\sim\pi_\theta(\cdot\mid x).
\]
Then the expected gradient of the per-prompt surrogate
satisfies
\begin{equation}
\begin{aligned}
\mathbb{E}
\left[
\nabla_\theta
\widehat{\mathcal{L}}_{\mathrm{on}}(x)
\right]
={}&
-\mu_{\mathrm{ref}}^{+}(x)
+
\left[
1-(1-Z_\theta(x))^{G}
\right]
\mu_{\theta}^{+}(x).
\end{aligned}
\label{eq:appb-finite-gradient}
\end{equation}
Equivalently,
\begin{equation}
\begin{aligned}
\mathbb{E}
\left[
\nabla_\theta
\widehat{\mathcal{L}}_{\mathrm{on}}(x)
\right]
={}&
\nabla_\theta\mathcal{L}_{\mathrm{CoKL}}(x)
-
(1-Z_\theta(x))^{G}
\nabla_\theta\log Z_\theta(x).
\end{aligned}
\label{eq:appb-finite-gradient-decomposition}
\end{equation}
Moreover, define
\begin{equation}
\begin{aligned}
R_G(Z)
&=
\sum_{m=G+1}^{\infty}
\frac{(1-Z)^m}{m}
\\
&=
-\log Z
-
\sum_{m=1}^{G}
\frac{(1-Z)^m}{m},
\qquad 0<Z\leq 1.
\end{aligned}
\label{eq:appb-recovery-term}
\end{equation}
Since
\begin{equation}
R_G'(Z)
=
-\frac{(1-Z)^G}{Z},
\label{eq:appb-recovery-derivative}
\end{equation}
Eq.~\ref{eq:appb-finite-gradient-decomposition} can be written as
\begin{equation}
\begin{aligned}
\mathbb{E}
\left[
\nabla_\theta
\widehat{\mathcal{L}}_{\mathrm{on}}(x)
\right]
=
\nabla_\theta
\left[
\mathcal{L}_{\mathrm{CoKL}}(x)
+
R_G(Z_\theta(x))
\right].
\end{aligned}
\label{eq:appb-recovery-gradient}
\end{equation}
The recovery term satisfies
\begin{equation}
R_G(Z)\geq 0,
\qquad
R_G'(Z)<0
\quad\text{for }0<Z<1,
\label{eq:appb-recovery-properties}
\end{equation}
and
\begin{equation}
\lim_{Z\rightarrow 1}R_G(Z)=0,
\qquad
\lim_{G\rightarrow\infty}R_G(Z)=0
\quad\text{for every fixed }Z>0.
\label{eq:appb-recovery-limits}
\end{equation}

\paragraph{Proof.}
Conditioned on \(K_\theta=k>0\), the \(k\)
verified-correct responses are distributed according to
\(\pi_\theta^{+}(\cdot\mid x)\). Therefore,
\begin{equation}
\mathbb{E}
\left[
\sum_{j=1}^{G}
\delta_j s_\theta(y'_j\mid x)
\;\middle|\;
K_\theta>0
\right]
=
\mu_\theta^{+}(x).
\label{eq:appb-conditional-current-score}
\end{equation}
When \(K_\theta=0\), all current-policy conditional weights
are set to zero. Since
\[
K_\theta
\sim
\operatorname{Binomial}(G,Z_\theta(x)),
\]
we have
\begin{equation}
\Pr(K_\theta>0)
=
1-(1-Z_\theta(x))^G.
\label{eq:appb-positive-count-probability}
\end{equation}
Consequently,
\begin{equation}
\begin{aligned}
\mathbb{E}
\left[
\sum_{j=1}^{G}
\delta_j s_\theta(y'_j\mid x)
\right]
=
\left[
1-(1-Z_\theta(x))^G
\right]
\mu_\theta^{+}(x).
\end{aligned}
\label{eq:appb-expected-current-score}
\end{equation}

For the reference-side term, conditioning the reference
group on containing at least one verified-correct response
gives
\begin{equation}
\mathbb{E}
\left[
\sum_{j=1}^{G}
\alpha_j s_\theta(y_j\mid x)
\;\middle|\;
K_{\mathrm{ref}}>0
\right]
=
\mu_{\mathrm{ref}}^{+}(x).
\label{eq:appb-expected-reference-score}
\end{equation}
Combining Eqs.~\ref{eq:appb-expected-current-score}
and~\ref{eq:appb-expected-reference-score} yields
Eq.~\ref{eq:appb-finite-gradient}. Using
\[
\nabla_\theta
\mathcal{L}_{\mathrm{CoKL}}(x)
=
-\mu_{\mathrm{ref}}^{+}(x)
+
\mu_\theta^{+}(x)
\]
gives Eq.~\ref{eq:appb-finite-gradient-decomposition}. Finally, differentiating Eq.~\ref{eq:appb-recovery-term} yields
\[
R_G'(Z)
=
-
\sum_{m=G+1}^{\infty}(1-Z)^{m-1}
=
-\frac{(1-Z)^G}{Z}.
\]
Therefore,
\[
\nabla_\theta R_G(Z_\theta(x))
=
-(1-Z_\theta(x))^G
\nabla_\theta\log Z_\theta(x),
\]
which proves Eq.~\ref{eq:appb-recovery-gradient}. The remaining
properties follow directly from the nonnegative series representation in
Eq.~\ref{eq:appb-recovery-term}.
\hfill\(\square\)

The finite-group recovery term does not anchor the current
correctness probability to \(Z_{\mathrm{ref}}(x)\). Instead,
because \(R_G'(Z)<0\), minimizing this term provides a
one-sided pressure toward higher correctness. To make this
distinction explicit, consider the expected finite-group
population objective
\begin{equation}
J_G(\pi)
=
Z_\pi
-
\beta
\left[
D_{\mathrm{KL}}
\left(
\pi_{\mathrm{ref}}^{+}
\|
\pi^{+}
\right)
+
R_G(Z_\pi)
\right].
\label{eq:appb-finite-population-objective}
\end{equation}
For fixed \(\pi^{+}\),
\begin{equation}
\frac{\partial J_G}{\partial Z_\pi}
=
1
+
\beta
\frac{(1-Z_\pi)^G}{Z_\pi}
>
0.
\label{eq:appb-finite-objective-derivative}
\end{equation}
Thus, increasing the total correctness probability strictly
improves the objective. Under the same unrestricted
policy-distribution setting used in Appendix~C, every global
maximizer therefore satisfies
\begin{equation}
Z_{\pi^\star}=1,
\qquad
\pi^{+\star}=\pi_{\mathrm{ref}}^{+}.
\label{eq:appb-finite-optimum}
\end{equation}
Hence, the finite-group fallback does not reintroduce the
strict correctness gap induced by full-policy KL regularization.

\subsection{Off-Policy Estimation}
For completeness, we derive an off-policy extension for settings
where responses generated by a behavior policy are reused. Suppose that the current-policy conditional term is instead estimated
using responses generated by a behavior policy
\(\pi_{\mathrm{old}}\):
\begin{equation}
\hat y_{ij}
\sim
\pi_{\mathrm{old}}(\cdot\mid x_i),
\qquad
\hat r_{ij}
=
r(x_i,\hat y_{ij}).
\label{eq:appb-old-samples}
\end{equation}
We assume that the support of the current policy is contained in the
support of the behavior policy for the responses considered below. For a complete response \(\hat y_{ij}\) of length \(T_{ij}\), define
its sequence log-likelihood under the current policy as
\begin{equation}
\begin{aligned}
\ell_\theta(\hat y_{ij}\mid x_i)
=
\sum_{t=1}^{T_{ij}}
\log\pi_\theta
\bigl(
\hat y_{ij,t}
\mid
x_i,\hat y_{ij,<t}
\bigr).
\end{aligned}
\label{eq:appb-sequence-logprob}
\end{equation}
The behavior-policy log-likelihood
\(\ell_{\mathrm{old}}(\hat y_{ij}\mid x_i)\) is defined analogously.
The sequence-level log-importance ratio is
\begin{equation}
\Delta_{ij}(\theta)
=
\ell_\theta(\hat y_{ij}\mid x_i)
-
\ell_{\mathrm{old}}(\hat y_{ij}\mid x_i).
\label{eq:appb-log-ratio}
\end{equation}
Hence, the sequence-level importance ratio is
\begin{equation}
\begin{aligned}
\rho_{ij}(\theta)
&=
\exp\bigl(\Delta_{ij}(\theta)\bigr)
=
\frac{
\pi_\theta(\hat y_{ij}\mid x_i)
}{
\pi_{\mathrm{old}}(\hat y_{ij}\mid x_i)
}.
\end{aligned}
\label{eq:appb-importance-ratio}
\end{equation}

For compactness, let
\(\mathbb{E}_{\mathrm{old},i}[\cdot]\) denote expectation under
\(\pi_{\mathrm{old}}(\cdot\mid x_i)\). The current correctness
probability can be expressed as
\begin{equation}
Z_\theta(x_i)
=
\mathbb{E}_{\mathrm{old},i}
\left[
\rho(\hat y;\theta)
r(x_i,\hat y)
\right].
\label{eq:appb-is-z}
\end{equation}
Its log-gradient is therefore
\begin{equation}
\begin{aligned}
\nabla_\theta\log Z_\theta(x_i)
=
\frac{
\mathbb{E}_{\mathrm{old},i}
\left[
\rho(\hat y;\theta)
r(x_i,\hat y)
\nabla_\theta\log\pi_\theta(\hat y\mid x_i)
\right]
}{
\mathbb{E}_{\mathrm{old},i}
\left[
\rho(\hat y;\theta)
r(x_i,\hat y)
\right]
}.
\end{aligned}
\label{eq:appb-is-logz-gradient}
\end{equation}
Using \(G\) behavior-policy samples, this gradient is approximated by
\begin{equation}
\begin{aligned}
\nabla_\theta\log Z_\theta(x_i)
\approx
\sum_{j=1}^G
\frac{
\rho_{ij}(\theta)\hat r_{ij}
}{
\sum_{k=1}^G
\rho_{ik}(\theta)\hat r_{ik}
}
\nabla_\theta
\log\pi_\theta(\hat y_{ij}\mid x_i).
\end{aligned}
\label{eq:appb-off-policy-gradient}
\end{equation}

\subsection{Clipped Off-Policy CoKL Surrogate}

Sequence-level importance ratios may have high variance. We therefore
define the clipped importance weight
\begin{equation}
\begin{aligned}
\tilde\rho_{ij}
=
\exp
\Bigl[
\operatorname{clip}
\bigl(
\Delta_{ij}(\theta),
\log(1-\epsilon_{\mathrm{IS}}),
\log(1+\epsilon_{\mathrm{IS}})
\bigr)
\Bigr].
\end{aligned}
\label{eq:appb-clipped-ratio}
\end{equation}
The corresponding self-normalized current-policy conditional weights are
\begin{equation}
\gamma_{ij}
=
\frac{
\tilde\rho_{ij}\hat r_{ij}
}{
\sum_{k=1}^G
\tilde\rho_{ik}\hat r_{ik}
}.
\label{eq:appb-gamma}
\end{equation}
When the denominator in Eq.~\ref{eq:appb-gamma} is zero, we set
\(\gamma_{ij}=0\) for all responses in that group. Since
the clipped importance weights are strictly positive, this
case occurs exactly when the behavior-policy rollout group
contains no verified-correct response. The current-policy
conditional correction is then unavailable, and the
surrogate falls back to the reference-correct term. The clipped off-policy CoKL surrogate is
\begin{equation}
\begin{aligned}
\widehat{\mathcal{L}}_{\mathrm{CoKL}}
={}&
-
\frac{1}{N}
\sum_{i=1}^N
\sum_{j=1}^G
\operatorname{sg}(\alpha_{ij})
\log\pi_\theta(y_{ij}\mid x_i)
\\
&+
\frac{1}{N}
\sum_{i=1}^N
\sum_{j=1}^G
\operatorname{sg}(\gamma_{ij})
\log\pi_\theta(\hat y_{ij}\mid x_i).
\end{aligned}
\label{eq:appb-practical-loss}
\end{equation}

Unlike the on-policy surrogate analyzed above, the clipped off-policy
surrogate is additionally affected by finite-sample self-normalization,
policy mismatch, and importance-ratio clipping, and therefore does not
admit the same exact recovery-term characterization. At the point of
sample collection, \(\theta=\theta_{\mathrm{old}}\) and
\(\tilde{\rho}_{ij}=1\), so it reduces to the on-policy surrogate.
During subsequent optimization epochs, the clipped importance weights
approximately correct for policy mismatch.

\section{Appendix C: Proof of Correctness-Mass Decoupling of CoKL}
\begin{theorem}[Correctness-Mass Decoupling under CoKL]
\label{thm:cokl_decoupling}
For $\beta>0$, consider the population-level objective
\[
J_{\mathrm{CoKL}}(\pi)
=
Z_\pi
-
\beta D_{\mathrm{KL}}
\left(
\pi_{\mathrm{ref}}^+
\Vert
\pi^+
\right).
\]
The objective depends on $\pi$ only through $Z_\pi$ and
$\pi^+$ and is independent of $\pi^-$. For any two policies
$\pi_1$ and $\pi_2$ satisfying
$\pi_1^+=\pi_2^+$ and $Z_{\pi_2}>Z_{\pi_1}$,
\[
J_{\mathrm{CoKL}}(\pi_2)
-
J_{\mathrm{CoKL}}(\pi_1)
=
Z_{\pi_2}-Z_{\pi_1}
>
0.
\]
Moreover, over the full probability simplex,
$\max_\pi J_{\mathrm{CoKL}}(\pi)=1$, and every global
maximizer satisfies
\[
Z_{\pi^\star}=1,
\qquad
\pi^{+\star}=\pi_{\mathrm{ref}}^+.
\]
\end{theorem}

\begin{proof}
We first show that $(z,p^+,p^-)$ provides a valid
parameterization of the unrestricted policy-distribution
space. For any policy $\pi$ satisfying $0<Z_\pi<1$, define
\[
z=Z_\pi
=
\sum_{y\in\mathcal{C}}\pi(y),
\]
and
\[
p^+(y)
=
\frac{\pi(y)}{z},
\quad y\in\mathcal{C},
\qquad
p^-(y)
=
\frac{\pi(y)}{1-z},
\quad y\in\mathcal{I}.
\]
These are valid conditional distributions because
\[
\sum_{y\in\mathcal{C}}p^+(y)=1,
\qquad
\sum_{y\in\mathcal{I}}p^-(y)=1.
\]
Moreover, the original policy can be recovered from the
triple $(z,p^+,p^-)$ as
\[
\pi(y)
=
\begin{cases}
z p^+(y), & y\in\mathcal{C},\\[1mm]
(1-z)p^-(y), & y\in\mathcal{I}.
\end{cases}
\]

Conversely, let $z\in(0,1)$, let $p^+$ be any probability
distribution on $\mathcal{C}$, and let $p^-$ be any
probability distribution on $\mathcal{I}$. Define
\[
\pi_{z,p^+,p^-}(y)
=
\begin{cases}
z p^+(y), & y\in\mathcal{C},\\[1mm]
(1-z)p^-(y), & y\in\mathcal{I}.
\end{cases}
\]
This construction yields a valid policy since
\[
\begin{aligned}
\sum_{y\in\mathcal{Y}}
\pi_{z,p^+,p^-}(y)
&=
z\sum_{y\in\mathcal{C}}p^+(y)
+
(1-z)\sum_{y\in\mathcal{I}}p^-(y)
\\
&=
z+(1-z)
\\
&=1.
\end{aligned}
\]
It also satisfies
\[
Z_{\pi_{z,p^+,p^-}}=z,
\qquad
\pi_{z,p^+,p^-}^+=p^+,
\qquad
\pi_{z,p^+,p^-}^-=p^-.
\]
Therefore, in the unrestricted policy-distribution space,
$z$, $p^+$, and $p^-$ can be varied separately through this
bijective parameterization. Using $Z_\pi=z$ and $\pi^+=p^+$, the CoKL objective becomes
\[
\begin{aligned}
J_{\mathrm{CoKL}}(z,p^+,p^-)
&=
Z_\pi
-
\beta D_{\mathrm{KL}}
\left(
\pi_{\mathrm{ref}}^+
\Vert
\pi^+
\right)
\\
&=
z
-
\beta D_{\mathrm{KL}}
\left(
\pi_{\mathrm{ref}}^+
\Vert
p^+
\right).
\end{aligned}
\]
The objective is therefore independent of $p^-$. Moreover,
for fixed $p^+$ and $p^-$ and any $0<z_1<z_2<1$,
\[
\begin{aligned}
&
J_{\mathrm{CoKL}}(z_2,p^+,p^-)
-
J_{\mathrm{CoKL}}(z_1,p^+,p^-)
\\
&=
\left[
z_2
-
\beta D_{\mathrm{KL}}
\left(
\pi_{\mathrm{ref}}^+
\Vert
p^+
\right)
\right]
\\
&\quad-
\left[
z_1
-
\beta D_{\mathrm{KL}}
\left(
\pi_{\mathrm{ref}}^+
\Vert
p^+
\right)
\right]
\\
&=
z_2-z_1
\\
&>0.
\end{aligned}
\]
Thus, increasing the total probability assigned to correct
responses while preserving the conditional distributions
strictly improves the objective without changing the CoKL
penalty. It remains to characterize the global optimum. For every
policy,
\[
Z_\pi\leq1
\]
and
\[
D_{\mathrm{KL}}
\left(
\pi_{\mathrm{ref}}^+
\Vert
\pi^+
\right)
\geq0.
\]
Consequently,
\[
J_{\mathrm{CoKL}}(\pi)
\leq
Z_\pi
\leq1.
\]
Now define the boundary policy
\[
\bar{\pi}(y)
=
\begin{cases}
\pi_{\mathrm{ref}}^+(y),
& y\in\mathcal{C},\\[1mm]
0,
& y\in\mathcal{I}.
\end{cases}
\]
Since $\pi_{\mathrm{ref}}^+$ is normalized over
$\mathcal{C}$,
\[
Z_{\bar{\pi}}=1,
\qquad
\bar{\pi}^+=\pi_{\mathrm{ref}}^+.
\]
Hence,
\[
D_{\mathrm{KL}}
\left(
\pi_{\mathrm{ref}}^+
\Vert
\bar{\pi}^+
\right)=0,
\]
and therefore
\[
J_{\mathrm{CoKL}}(\bar{\pi})=1.
\]
Thus,
\[
\max_\pi J_{\mathrm{CoKL}}(\pi)=1.
\]

Finally, let $\pi^\star$ be any global maximizer. Since
$J_{\mathrm{CoKL}}(\pi^\star)=1$, while
$Z_{\pi^\star}\leq1$ and the KL divergence is nonnegative,
equality requires
\[
Z_{\pi^\star}=1
\]
and
\[
D_{\mathrm{KL}}
\left(
\pi_{\mathrm{ref}}^+
\Vert
\pi^{+\star}
\right)=0.
\]
A KL divergence is zero if and only if its two arguments are
identical. Therefore,
\[
\pi^{+\star}
=
\pi_{\mathrm{ref}}^+.
\]
\end{proof}

\section{Appendix D: Proof of the Strict Correctness Gap under Full-Policy KL}
\begin{theorem}[Strict Correctness Gap Induced by Full-Policy KL]
\label{thm:full_kl_gap}
Let $q=Z_{\mathrm{ref}}\in(0,1)$ and $\beta>0$. Consider
the full-policy forward- and reverse-KL objectives
\[
\begin{aligned}
J_{\mathrm{FKL}}(\pi)
&=
Z_\pi
-
\beta D_{\mathrm{KL}}
\left(
\pi_{\mathrm{ref}}\Vert\pi
\right),
\\
J_{\mathrm{RKL}}(\pi)
&=
Z_\pi
-
\beta D_{\mathrm{KL}}
\left(
\pi\Vert\pi_{\mathrm{ref}}
\right).
\end{aligned}
\]
Under the required absolute-continuity conditions, their
global optimizers satisfy
$\pi^{+\star}=\pi_{\mathrm{ref}}^+$ and
$\pi^{-\star}=\pi_{\mathrm{ref}}^-$. The corresponding
optimal correctness probabilities are
\[
\begin{aligned}
Z_{\mathrm{FKL}}^\star
&=
\frac{
1-\beta+
\sqrt{(\beta-1)^2+4\beta q}
}{2},
\\
Z_{\mathrm{RKL}}^\star
&=
\sigma\left(
\operatorname{logit}(q)+\frac{1}{\beta}
\right),
\end{aligned}
\]
where $\sigma(a)=(1+e^{-a})^{-1}$. Both satisfy
$q<Z^\star<1$ and are strictly decreasing in $\beta$.
Furthermore, for
$\diamond\in\{\mathrm{FKL},\mathrm{RKL}\}$,
\[
\lim_{\beta\to0^+}Z_\diamond^\star=1,
\qquad
\lim_{\beta\to\infty}Z_\diamond^\star=q.
\]
\end{theorem}

\begin{proof}
By the parameterization established in
Appendix C, every policy with
$0<Z_\pi<1$ can be represented as
\[
(z,p^+,p^-)
=
(Z_\pi,\pi^+,\pi^-).
\]
The reference policy is correspondingly represented by
\[
\left(
q,
\pi_{\mathrm{ref}}^+,
\pi_{\mathrm{ref}}^-
\right).
\]

\paragraph{Full forward KL.}
The conditional decomposition of the forward KL gives
\[
\begin{aligned}
&D_{\mathrm{KL}}
\left(
\pi_{\mathrm{ref}}
\Vert
\pi
\right)
=
D_{\mathrm{KL}}
\left(
\operatorname{Bern}(q)
\Vert
\operatorname{Bern}(z)
\right)
\\
&\quad+
qD_{\mathrm{KL}}
\left(
\pi_{\mathrm{ref}}^+
\Vert
p^+
\right)
+
(1-q)D_{\mathrm{KL}}
\left(
\pi_{\mathrm{ref}}^-
\Vert
p^-
\right).
\end{aligned}
\]
The corresponding objective is therefore
\[
\begin{aligned}
J_{\mathrm{FKL}}(z,p^+,p^-)
&=
z
-
\beta D_{\mathrm{KL}}
\left(
\operatorname{Bern}(q)
\Vert
\operatorname{Bern}(z)
\right)
\\
&\quad-
\beta qD_{\mathrm{KL}}
\left(
\pi_{\mathrm{ref}}^+
\Vert
p^+
\right)
\\
&\quad-
\beta(1-q)D_{\mathrm{KL}}
\left(
\pi_{\mathrm{ref}}^-
\Vert
p^-
\right).
\end{aligned}
\]
For any fixed $z$, the two conditional KL terms are
nonnegative and attain their minimum value of zero at
\[
p^+=\pi_{\mathrm{ref}}^+,
\qquad
p^-=\pi_{\mathrm{ref}}^-.
\]
The distributional optimization thus reduces to
\[
\max_{0<z<1}f(z),
\]
where
\[
\begin{aligned}
f(z)
&=
z
-
\beta D_{\mathrm{KL}}
\left(
\operatorname{Bern}(q)
\Vert
\operatorname{Bern}(z)
\right)
\\
&=
z
-
\beta
\left[
q\log\frac{q}{z}
+
(1-q)\log\frac{1-q}{1-z}
\right].
\end{aligned}
\]

The first and second derivatives of $f$ are
\[
\begin{aligned}
f'(z)
&=
1
-
\beta
\frac{z-q}{z(1-z)},
\\
f''(z)
&=
-\beta
\left[
\frac{q}{z^2}
+
\frac{1-q}{(1-z)^2}
\right]
<0.
\end{aligned}
\]
Therefore, $f$ is strictly concave on $(0,1)$. Moreover, the
forward Bernoulli KL diverges as $z\to0^+$ or $z\to1^-$, so
the unique global maximum is attained at an interior
stationary point.

Setting $f'(z)=0$ gives
\[
1
=
\beta
\frac{z-q}{z(1-z)},
\]
or equivalently,
\[
z^2+(\beta-1)z-\beta q=0.
\]
The two roots are
\[
z
=
\frac{
1-\beta
\pm
\sqrt{(\beta-1)^2+4\beta q}
}{2}.
\]
Their product is $-\beta q<0$, so exactly one root is
positive. The unique feasible solution is therefore
\[
Z_{\mathrm{FKL}}^\star
=
\frac{
1-\beta+
\sqrt{(\beta-1)^2+4\beta q}
}{2}.
\]

To determine its location, note that
\[
f'(q)=1>0,
\qquad
\lim_{z\to1^-}f'(z)=-\infty.
\]
Since $f'$ is strictly decreasing, its unique zero satisfies
\[
q<Z_{\mathrm{FKL}}^\star<1.
\]

The stationary-point equation can be rearranged as
\[
\beta
=
h\left(Z_{\mathrm{FKL}}^\star\right),
\qquad
h(z)
=
\frac{z(1-z)}{z-q},
\quad z\in(q,1).
\]
Its derivative is
\[
h'(z)
=
-
\frac{
(z-q)^2+q(1-q)
}{
(z-q)^2
}
<0.
\]
Thus, $h$ is strictly decreasing, and so is its inverse
$Z_{\mathrm{FKL}}^\star(\beta)$. Furthermore,
\[
\lim_{z\to q^+}h(z)=\infty,
\qquad
\lim_{z\to1^-}h(z)=0,
\]
which implies
\[
\lim_{\beta\to0^+}
Z_{\mathrm{FKL}}^\star=1,
\qquad
\lim_{\beta\to\infty}
Z_{\mathrm{FKL}}^\star=q.
\]

\paragraph{Full reverse KL.}
The reverse KL admits the analogous decomposition
\[
\begin{aligned}
&D_{\mathrm{KL}}
\left(
\pi
\Vert
\pi_{\mathrm{ref}}
\right)
\\
&=
D_{\mathrm{KL}}
\left(
\operatorname{Bern}(z)
\Vert
\operatorname{Bern}(q)
\right)
\\
&\quad+
zD_{\mathrm{KL}}
\left(
p^+
\Vert
\pi_{\mathrm{ref}}^+
\right)
\\
&\quad+
(1-z)D_{\mathrm{KL}}
\left(
p^-
\Vert
\pi_{\mathrm{ref}}^-
\right).
\end{aligned}
\]
For fixed $z$, the two conditional KL terms are minimized at
\[
p^+=\pi_{\mathrm{ref}}^+,
\qquad
p^-=\pi_{\mathrm{ref}}^-.
\]
Hence, the optimization reduces to
\[
\max_{0\leq z\leq1}g(z),
\]
where
\[
\begin{aligned}
g(z)
&=
z
-
\beta D_{\mathrm{KL}}
\left(
\operatorname{Bern}(z)
\Vert
\operatorname{Bern}(q)
\right)
\\
&=
z
-
\beta
\left[
z\log\frac{z}{q}
+
(1-z)\log\frac{1-z}{1-q}
\right].
\end{aligned}
\]

Its derivatives are
\[
\begin{aligned}
g'(z)
&=
1
-
\beta
\log
\frac{
z(1-q)
}{
q(1-z)
},
\\
g''(z)
&=
-\frac{\beta}{z(1-z)}
<0.
\end{aligned}
\]
Thus, $g$ is strictly concave on $(0,1)$. In addition,
\[
\lim_{z\to0^+}g'(z)=+\infty,
\qquad
\lim_{z\to1^-}g'(z)=-\infty,
\]
so it has a unique interior global maximizer.

Setting $g'(z)=0$ gives
\[
\log
\frac{
z(1-q)
}{
q(1-z)
}
=
\frac{1}{\beta}.
\]
Equivalently,
\[
\operatorname{logit}(z)
=
\operatorname{logit}(q)
+
\frac{1}{\beta},
\]
and hence
\[
Z_{\mathrm{RKL}}^\star
=
\sigma
\left(
\operatorname{logit}(q)+\frac{1}{\beta}
\right).
\]
Because $\beta^{-1}>0$ and $\sigma$ is strictly increasing,
\[
Z_{\mathrm{RKL}}^\star
>
\sigma(\operatorname{logit}(q))
=
q.
\]
Since the sigmoid argument is finite for every finite
$\beta>0$,
\[
Z_{\mathrm{RKL}}^\star<1.
\]
Therefore,
\[
q<Z_{\mathrm{RKL}}^\star<1.
\]

Differentiating the closed-form solution gives
\[
\frac{
\partial Z_{\mathrm{RKL}}^\star
}{
\partial\beta
}
=
-
\frac{
Z_{\mathrm{RKL}}^\star
\left(
1-Z_{\mathrm{RKL}}^\star
\right)
}{
\beta^2
}
<0.
\]
Finally,
\[
\lim_{\beta\to0^+}
\left(
\operatorname{logit}(q)+\frac{1}{\beta}
\right)
=
+\infty,
\]
whereas
\[
\lim_{\beta\to\infty}
\left(
\operatorname{logit}(q)+\frac{1}{\beta}
\right)
=
\operatorname{logit}(q).
\]
Applying the sigmoid function yields
\[
\lim_{\beta\to0^+}
Z_{\mathrm{RKL}}^\star=1,
\qquad
\lim_{\beta\to\infty}
Z_{\mathrm{RKL}}^\star=q.
\]
This completes the proof.
\end{proof}

\section{Appendix E: Experimental Details}
\subsection{E.1 Details of the Controlled Multi-Solution RL Experiment}
The environment contains \(48\) latent clusters and \(200\) candidate
actions, of which \(12\) are correct for each cluster. Cluster centers are
sampled in \(\mathbb{R}^{64}\), and inputs are generated as
\(x=\mu_z+1.4\epsilon\), where \(\epsilon\sim\mathcal{N}(0,I)\). The oracle
policy assigns a total probability mass of \(0.10\) to the correct-action
set. Its correctness-conditioned distribution is sampled from a Dirichlet
distribution with concentration \(50\), centered at a Zipf prior whose
exponents are linearly spaced over \([1.0,2.4]\); the Zipf ranks are randomly
assigned to the correct actions, and the remaining probability mass is
distributed uniformly over incorrect actions. We generate \(8{,}000\)
training inputs and \(4{,}000\) test inputs. A
\(64\)-\(128\)-\(128\)-\(200\) MLP with Tanh activations is pretrained for
\(1{,}000\) steps using soft-label cross-entropy and then frozen as the
reference policy. All methods are initialized from this model and trained
for \(1{,}000\) steps using the exact expected reward and exact
regularization, with Adam, batch size \(512\), learning rate
\(1.2\times10^{-3}\), and maximum gradient norm \(1.0\). Results are averaged
over 5 random seeds and evaluated every \(100\) steps. For each
regularized method, the coefficient is swept over
\(\beta\in\{0.0005,\allowbreak 0.001,\allowbreak 0.002,\allowbreak
0.003,\allowbreak 0.005,\allowbreak 0.007,\allowbreak 0.01,\allowbreak
0.02,\allowbreak 0.03,\allowbreak 0.05,\allowbreak 0.07,\allowbreak
0.1,\allowbreak 0.2,\allowbreak 0.3\}\), while unregularized RL uses
\(\beta=0\).

\subsection{E.2 Datasets for the Main Experiments}
We summarize the number of samples in the training and evaluation datasets in Table~\ref{tab:datasets_statics}.

\begin{table}[t]
\centering
\caption{Statistics of the training and test datasets used by all methods.}
\label{tab:datasets_statics}
\small
\setlength{\tabcolsep}{4pt}
\renewcommand{\arraystretch}{1.1}
\begin{tabular}{lll}
\toprule
\textbf{Dataset} & \textbf{Description
} & \textbf{\# Train / Test} \\
\midrule
MATH-Train            & Math Reasoning     & 50000 / - \\
MATH-Val            & Math Reasoning     & - / 1000 \\
MATH500            & Math Reasoning     & - / 500 \\
Olympiad           & Math Reasoning     & - / 674 \\
Chat-Train            & Chat and Instruction Following     & 7544 / - \\
Chat-Val            & Chat and Instruction Following     & - / 250 \\
\bottomrule
\end{tabular}
\end{table}

\subsection{E.3 Training and Comparison Details}
For the KL-based methods, after math-task training, we use the resulting model to construct a KL buffer by randomly sampling 512 prompts from the math training set. Full Forward-KL and Full Reverse-KL use all 512 prompts, whereas CoKL and the two correct-only baselines use the same subset obtained by filtering out prompts for which none of the $G_{\mathrm{ref}}=8$ reference responses is verified as correct. During chat-task training, we randomly sample one batch of math data from this buffer to compute the KL loss, while no KL constraint is imposed on the chat data itself. We train all models for \(200\) steps on the math task. For the chat task,
Qwen3-0.6B is trained for \(800\) steps due to its slower convergence,
whereas Qwen3-1.7B and Qwen3-4B are each trained for \(600\) steps. We set the chat-task batch size to 64. For the KL data, the prompt batch size is set to 16 for the baseline methods and to 8 for CoKL, whose paired reference-policy and current-policy construction yields 16 response groups per update. The mini-batch size is set to 64 for GRPO without KL and to 80 for the KL-based methods. The GRPO rollout group size is set to \(G=8\) for all methods,
and the rollout group size for the KL regularization prompts is
also set to 8 for all KL-based methods. For CoKL, we use
\(G_{\mathrm{ref}}=8\) cached reference responses and
\(G=8\) fresh current-policy rollouts for each regularization
prompt. The current-policy conditional term is therefore
estimated on-policy at each regularization update. We set the
CoKL coefficient to \(\beta=1\). To ensure a fair comparison, we conduct
hyperparameter searches for all KL-based baselines using
Qwen3-1.7B. Specifically, the KL coefficient is selected from $\beta \in \{0.1, 1, 10\}$ for the Forward-KL variants and from $\beta \in \{10^{-3}, 10^{-2}, 1\}$ for the Reverse-KL variants. The corresponding raw scores are presented in
Figure~\ref{fig:kl_beta_search}. The Forward-KL and Reverse-KL families achieve their best average performance at $\beta=1$ and $\beta=10^{-3}$, respectively. Deviating from these values tends either to overconstrain adaptation to the new task or to exacerbate forgetting of previously acquired capabilities. Accordingly, we set $\beta=1$ for all Forward-KL variants and $\beta=10^{-3}$ for all Reverse-KL variants, and report the results obtained under their respective best-performing configurations. Our framework is implemented on top of VeRL~\cite{sheng2024hybridflow}. The learning rate is set to \(1\times10^{-6}\). All experiments were conducted on 8$\times$H20 GPUs with 141GB of memory, using Python 3.10 and PyTorch 2.6. All models are optimized with the AdamW optimizer (\(\beta_1 = 0.9\), \(\beta_2 = 0.95\), weight decay 0.01) and accelerated via vLLM~\cite{kwon2023efficient}.

\begin{figure}[t]
    \centering
    \begin{minipage}[t]{0.48\linewidth}
        \centering
        \includegraphics[width=\linewidth]{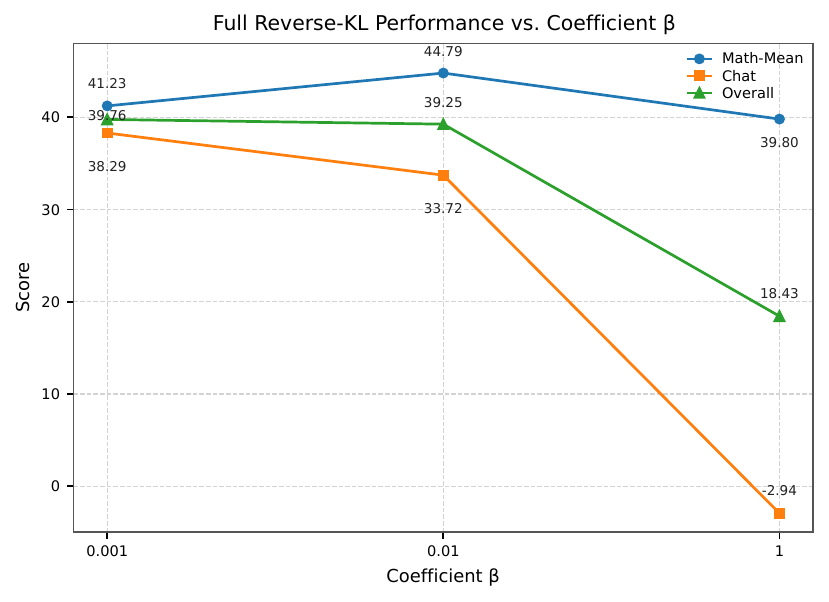}
    \end{minipage}
    \hfill
    \begin{minipage}[t]{0.48\linewidth}
        \centering
        \includegraphics[width=\linewidth]{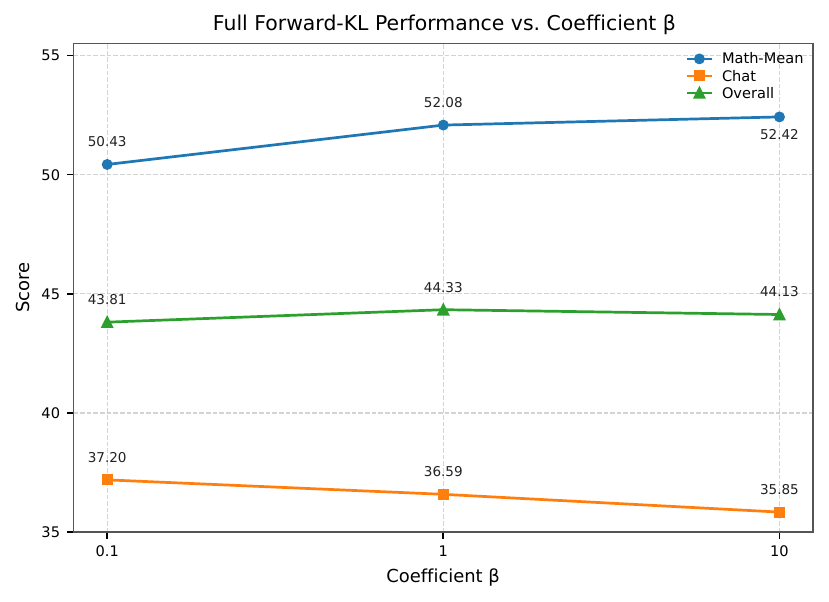}
    \end{minipage}

    \vspace{2mm}

    \begin{minipage}[t]{0.48\linewidth}
        \centering
        \includegraphics[width=\linewidth]{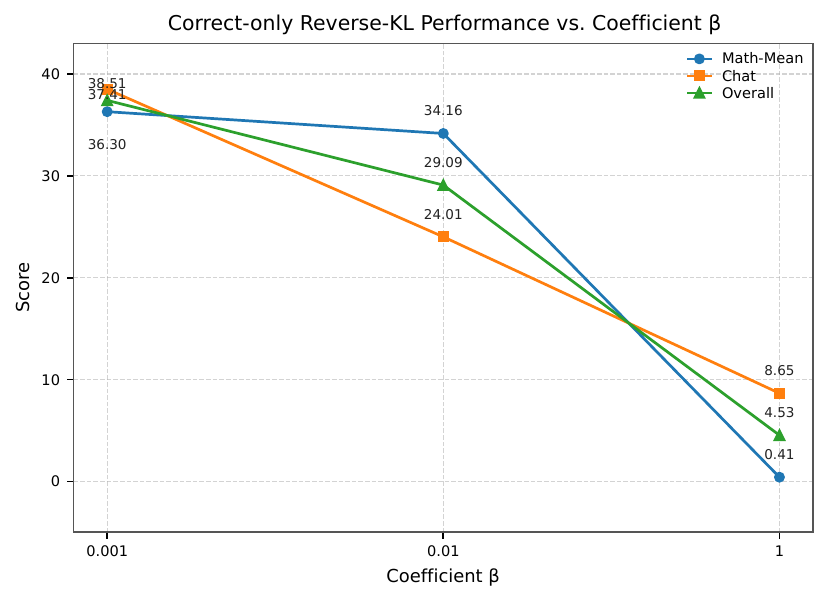}
    \end{minipage}
    \hfill
    \begin{minipage}[t]{0.48\linewidth}
        \centering
        \includegraphics[width=\linewidth]{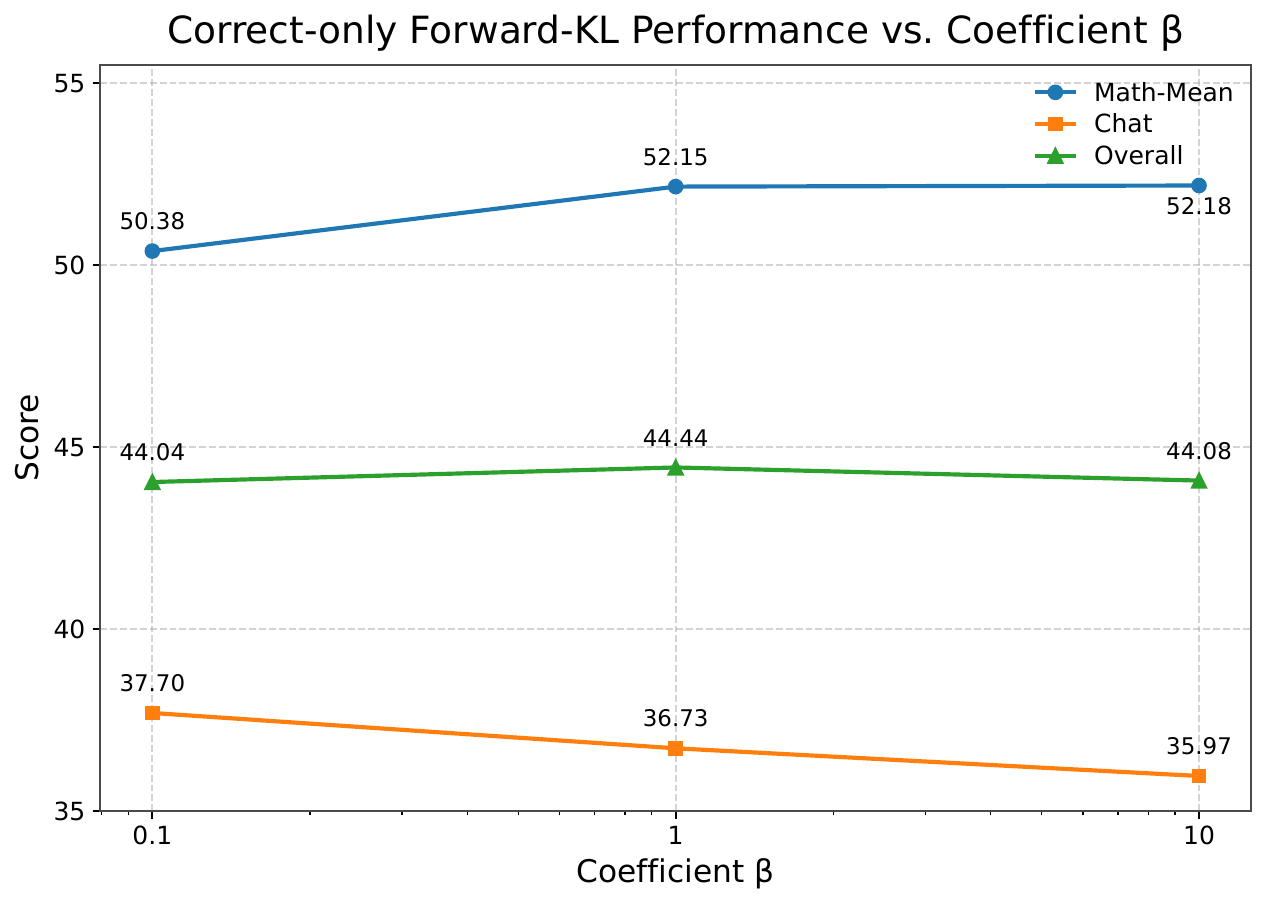}
    \end{minipage}

    \caption{Hyperparameter sensitivity of the baseline methods on Qwen3-1.7B under different KL coefficients $\beta$.}
    \label{fig:kl_beta_search}
\end{figure}

\subsection{E.4 Frequency of Zero-Correct Rollout Groups}

To quantify how often the finite-group fallback is activated, we measure
the fraction of current-policy rollout groups on the mathematical
regularization buffer that contain no verified-correct response. For a
rollout group of size \(G\), we define
\[
\widehat{p}_{0}
=
\frac{1}{N_{\mathrm{group}}}
\sum_i
\mathbb{I}
\left[
\sum_{j=1}^{G} r_{i,j}=0
\right],
\]
where \(N_{\mathrm{group}}\) is the total number of rollout groups
collected during training. As shown in Table~\ref{tab:zero-correct-groups}, most rollout groups use
the full two-term CoKL update. Zero-correct groups automatically activate
the reference-correct fallback analyzed in Appendix~B, with the activation
rate decreasing from \(16.08\%\) for Qwen3-0.6B to \(3.17\%\) for
Qwen3-4B.

\begin{table}[t]
\centering
{\small
\setlength{\tabcolsep}{7pt}
\begin{tabular}{lcc}
\toprule
\textbf{Model}
& \textbf{Zero-Correct}
& \textbf{At Least One Correct} \\
& \textbf{Groups}
& \textbf{Response} \\
\midrule
Qwen3-0.6B & 16.08\% & 83.92\% \\
Qwen3-1.7B & 12.71\% & 87.29\% \\
Qwen3-4B   & 3.17\%  & 96.83\% \\
\bottomrule
\end{tabular}
}
\caption{Fraction of current-policy rollout groups containing no
verified-correct response during CoKL training.}
\label{tab:zero-correct-groups}
\end{table}

\subsection{E.5 Dual-Anchor Normalization}
Because mathematical reasoning accuracy and chat reward are measured on
different numerical scales, directly averaging their raw values may
affect the aggregate comparison. We therefore additionally assess
aggregation robustness using dual-anchor normalization. The normalization
is performed separately within each model scale. For each mathematical reasoning benchmark \(d\), let \(M_{m,d}\) denote
the original score of method \(m\). We define its normalized score as
\begin{equation}
\widetilde{M}_{m,d}
=
100
\frac{
M_{m,d}-M_{\mathrm{GRPO},d}
}{
M_{\mathrm{Base},d}-M_{\mathrm{GRPO},d}
},
\end{equation}
where Base and GRPO w/o KL are assigned normalized scores of \(100\)
and \(0\), respectively. The normalized mathematical score is computed
by averaging the three benchmark-level scores:
\begin{equation}
\widetilde{M}_{m}
=
\frac{1}{3}
\sum_{d\in\mathcal{D}_{\mathrm{math}}}
\widetilde{M}_{m,d},
\end{equation}
where $\mathcal{D}_{\mathrm{math}}
=
\{
\mathrm{MATH\text{-}Val},
\mathrm{MATH500},
\mathrm{Olympiad}
\}$. For chat performance, let \(C_m\) denote the original Chat-Val score.
We reverse the anchor direction and define
\begin{equation}
\widetilde{C}_{m}
=
100
\frac{
C_m-C_{\mathrm{Base}}
}{
C_{\mathrm{GRPO}}-C_{\mathrm{Base}}
},
\end{equation}
such that Base and GRPO w/o KL receive normalized chat scores of \(0\)
and \(100\), respectively. The normalized aggregate score is
\begin{equation}
\widetilde{S}_{m}
=
\frac{
\widetilde{M}_{m}+\widetilde{C}_{m}
}{2}.
\end{equation}

This normalization measures each method's relative ability to retain
prior mathematical capabilities and learn the new chat task with respect
to the Base and GRPO anchors, placing the two objectives on a common
scale. We also report the corresponding unnormalized scores in
Table~\ref{tab:dual-anchor-original-results}. Two caveats apply when interpreting these unnormalized scores. First,
since math accuracy and chat reward lie on incommensurable scales,
directly averaging them lets the higher-variance metric dominate and
compresses genuine differences in the retention--adaptation trade-off.
Second, the achievable gain is bounded by the severity of forgetting:
the Math Avg.\ drop of GRPO w/o KL shrinks from $91.0\%$ (0.6B) to
$40.2\%$ (1.7B) and $19.7\%$ (4B), so larger models leave less headroom
for any preservation method, and all regularized methods converge at
the 4B scale. The advantage of CoKL is therefore most pronounced under
high forgetting pressure, and naturally reduces toward correct-only
replay when forgetting is mild.

\begin{table*}[t]
\centering
\small
\renewcommand{\arraystretch}{1.12}
\setlength{\tabcolsep}{4pt}
\begin{tabular*}{\textwidth}
{@{\extracolsep{\fill}}lcccccc}
\toprule
\textbf{Method}
& \multicolumn{4}{c}{\textbf{Mathematical Reasoning}}
& \multicolumn{1}{c}{\textbf{Chat}}
& \multicolumn{1}{c}{\textbf{Aggregate}} \\
\cmidrule(lr){2-5}
\cmidrule(lr){6-6}
\cmidrule(lr){7-7}
& \textbf{MATH-Val}
& \textbf{MATH500}
& \textbf{Olympiad}
& \textbf{Math Avg.}
& \textbf{Chat-Val}
& \textbf{Overall} \\
\midrule

\multicolumn{7}{c}{\textit{Qwen3-0.6B}} \\
\addlinespace[2pt]

Base (after Math training)
& 32.46 & 70.80 & 34.46 & 45.91
& -12.90 & 16.51 \\

GRPO w/o KL
& 2.44 & 6.25 & 3.65 & 4.11
& 73.98 & 39.04 \\

Full Reverse-KL
& \textbf{22.49}
& 59.85
& 20.72
& 34.35
& 61.14
& 47.75 \\

Full Forward-KL
& 20.35
& 63.67
& 23.41
& 35.81
& 67.55
& 51.68 \\

Correct-only Reverse-KL
& 16.27
& 40.20
& 17.80
& 24.76
& 62.95
& 43.86 \\

Correct-only Forward-KL
& \underline{21.87}
& \textbf{64.90}
& \textbf{23.70}
& \textbf{36.82}
& \underline{69.33}
& \underline{53.08} \\

\textbf{CoKL (ours)}
& 21.04
& \underline{64.57}
& \underline{23.65}
& \underline{36.42}
& \textbf{74.44}
& \textbf{55.43} \\

\midrule

\multicolumn{7}{c}{\textit{Qwen3-1.7B}} \\
\addlinespace[2pt]

Base (after Math training)
& 44.67 & 83.60 & 47.40 & 58.56
& -8.85 & 24.86 \\

GRPO w/o KL
& 18.77 & 64.00 & 22.31 & 35.03
& 37.45 & 36.24 \\

Full Reverse-KL
& 27.36
& 69.40
& 26.92
& 41.23
& \underline{38.29}
& 39.76 \\

Full Forward-KL
& \textbf{41.61}
& 79.83
& 34.79
& 52.08
& 36.59
& 44.34 \\

Correct-only Reverse-KL
& 20.19
& 64.67
& 24.04
& 36.30
& \textbf{38.51}
& 37.41 \\

Correct-only Forward-KL
& 41.22
& \underline{80.02}
& \textbf{35.21}
& \underline{52.15}
& 36.73
& \underline{44.44} \\

\textbf{CoKL (ours)}
& \underline{41.59}
& \textbf{80.98}
& \underline{35.00}
& \textbf{52.52}
& 38.12
& \textbf{45.32} \\

\midrule

\multicolumn{7}{c}{\textit{Qwen3-4B}} \\
\addlinespace[2pt]

Base (after Math training)
& 60.27 & 91.37 & 57.62 & 69.75
& -5.12 & 32.32 \\

GRPO w/o KL
& 43.09 & 85.90 & 39.02 & 56.00
& 49.09 & 52.55 \\

Full Reverse-KL
& 11.06
& 48.45
& 11.26
& 23.59
& 37.30
& 30.45 \\

Full Forward-KL
& \textbf{57.54}
& \underline{88.05}
& \underline{41.32}
& \textbf{62.30}
& 50.03
& 56.17 \\

Correct-only Reverse-KL
& 45.55
& 85.70
& 39.24
& 56.83
& 43.11
& 49.97 \\

Correct-only Forward-KL
& \underline{57.50}
& 87.70
& \textbf{41.34}
& 62.18
& \underline{51.07}
& \underline{56.62} \\

\textbf{CoKL (ours)}
& 57.36
& \textbf{88.28}
& 41.02
& \underline{62.22}
& \textbf{51.27}
& \textbf{56.74} \\

\bottomrule
\end{tabular*}

\caption{Original evaluation results used for dual-anchor normalization.
Math Avg.\ is the average of MATH-Val, MATH500, and Olympiad, while
Overall is the equally weighted mean of Math Avg.\ and Chat-Val. The
best and second-best results among the regularized post-training methods
within each model scale are shown in \textbf{bold} and
\underline{underline}, respectively.}
\label{tab:dual-anchor-original-results}
\end{table*}

\subsection{E.6 Training Dynamics}
To better understand how different objectives affect the optimization process, we compare the training dynamics of different methods on the Qwen3-0.6B model. Specifically, we track the validation performance on the chat and math tasks throughout training, using chat reward and math accuracy as the corresponding evaluation metrics. As shown in Figure~\ref{fig:training_dynamics_06b}, GRPO w/o KL achieves strong improvements on the chat task, but its math performance deteriorates substantially as training progresses. In contrast, CoKL achieves chat performance comparable to GRPO w/o KL while preserving considerably stronger performance on the math task. Consequently, CoKL provides the best overall balance between the two domains and achieves the highest average performance, further validating the effectiveness of our method.

\begin{figure}[t]
    \centering
    \includegraphics[width=\columnwidth]
    {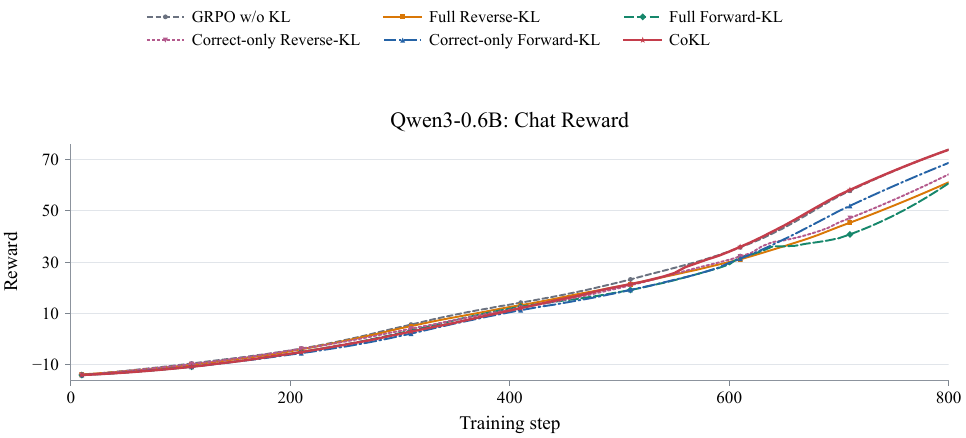}

    \vspace{2mm}

    \includegraphics[width=\columnwidth]
    {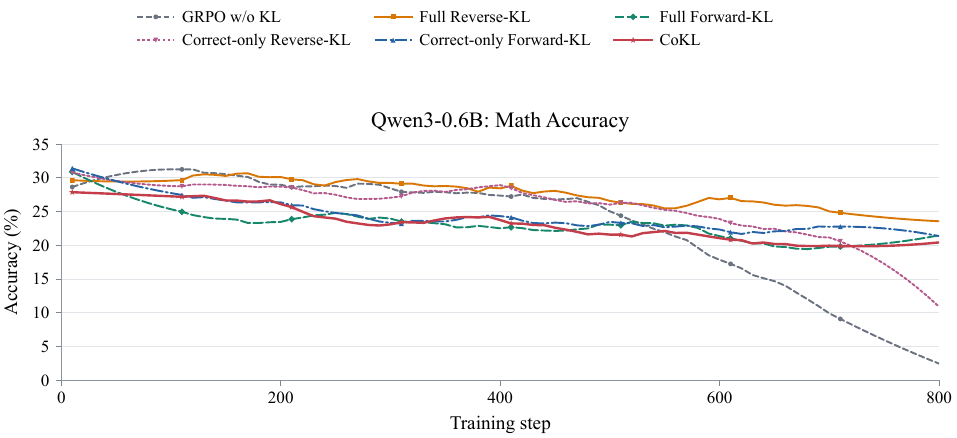}

    \caption{Training dynamics of different methods on Qwen3-0.6B.}
    \label{fig:training_dynamics_06b}
\end{figure}

\subsection{Appendix F: Additional Experimental Results}
\subsection{F.1 Comparison of Computational Cost}
To compare the training overhead, we report the wall-clock time of the first
training step using Qwen3-0.6B on 8$\times$H20 GPUs, as shown in
Table~\ref{tab:time_cost}. The Forward-KL variants require rollouts to be
generated in advance, whereas the Reverse-KL variants perform additional
rollouts during training. Moreover, the math data generally produce longer
responses. Consequently, these KL-based methods incur additional training
cost compared with GRPO. Nevertheless, the computational cost of CoKL
remains comparable to that of the other KL-based methods, while providing
meaningful performance improvements. Therefore, its modest additional
overhead is worthwhile.

\begin{table}[t]
\centering
\caption{Wall-clock time of the first training step.}
\label{tab:time_cost}
\small
\setlength{\tabcolsep}{3pt}
\renewcommand{\arraystretch}{1.1}
\begin{tabular}{lrr}
\toprule
\textbf{Method} & \textbf{Time (s)} & \textbf{Avg. Length} \\
\midrule
GRPO               & 204.9 & 929   \\
Correct-only Forward-KL         & 314.2 & 2,155 \\
Full Forward-KL     & 309.9 & 2,172 \\
Full Reverse-KL         & 293.4 & 1,788 \\
Correct-only Reverse-KL & 308.8 & 1,814 \\
CoKL (Ours)        & 323.4 & 2,136 \\
\bottomrule
\end{tabular}
\end{table}

\subsection{F.2 Reward-Blind Validation of Chat Evaluation}
To validate the reliability of the chat reward as an evaluation signal, we conduct a reward-blind assessment on 240 generations from 60 non-overlapping WildChat-IF prompts. GPT-5.6 Sol serves as a separate evaluator and rates each response using a five-point rubric covering instruction following and response correctness. Model identities and reward scores are hidden during evaluation to prevent potential bias. The chat reward exhibits a positive correlation with the independent ratings (Spearman's $\rho=0.399$; 95\% prompt-cluster bootstrap CI $[0.253,0.531]$) and agrees with the evaluator-preferred response in 74.8\% of non-tied within-prompt comparisons (95\% CI $[66.3\%,82.3\%]$). Moreover, the reward is negatively correlated with response length ($\rho=-0.342$), suggesting that its judgments are not driven by a simple preference for verbose responses. These results indicate that the chat reward provides a meaningful external signal for comparing response quality.


\section{Appendix G: Use of Large Language Models}
Some parts of the text were refined with the assistance of large language models (LLMs). All content and responsibility for the work remain with the authors.

\section{Appendix H: Prompts}
We use the prompts shown in Figure~\ref{fig:prompt-templates}.

\begin{figure*}[!t]
    \centering
    \begin{minipage}{0.98\textwidth}

    \begin{promptbox}{Math Prompt}
    \small
    \ttfamily
    \raggedright
    \setlength{\parindent}{0pt}

    Solve the following math problem. Please reason step by step, and put
    your final answer within \textbackslash boxed\{\}.

    \vspace{1.5mm}

    \{problem\}
    \end{promptbox}

    \vspace{2mm}

    \begin{promptbox}{Chat Prompt}
    \small
    \ttfamily
    \raggedright
    \setlength{\parindent}{0pt}

    \{raw user question or conversation context\}
    \end{promptbox}
    \end{minipage}

    \caption{Prompt templates used for the math and chat
    tasks.}
    \label{fig:prompt-templates}
\end{figure*}

\section{Appendix I: Pseudocode of CoKL}
We present the pseudocode of CoKL in Algorithm~\ref{alg:cokl-rl}.

\begin{algorithm}[t]
\caption{Correctness-Conditioned KL Regularization (CoKL)}
\label{alg:cokl-rl}
\begin{algorithmic}[1]
\REQUIRE Frozen reference policy \(\pi_{\mathrm{ref}}\); initial policy \(\pi_\theta\); current-task reward function \(R\); verifier \(r\); current-task prompt set \(\mathcal{D}_{\mathrm{RL}}\); regularization prompt set \(\mathcal{D}_{\mathrm{KL}}\); rollout sizes \(G_{\mathrm{ref}},G\); importance-ratio clipping parameter \(\epsilon_{\mathrm{IS}}\); update epochs \(\mu\); coefficient \(\beta\)
\STATE \(\mathcal{D}^{+}\leftarrow\emptyset\)
\FOR{each prompt \(x\in\mathcal{D}_{\mathrm{KL}}\)}
    \STATE Sample \(G_{\mathrm{ref}}\) reference responses from \(\pi_{\mathrm{ref}}(\cdot\mid x)\) and evaluate them with \(r\)
    \IF{at least one reference response is verified as correct}
        \STATE Add the prompt and its reference responses to \(\mathcal{D}^{+}\)
    \ENDIF
\ENDFOR
\FOR{iteration \(=1,\ldots,I\)}
    \STATE Sample a current-task minibatch \(\mathcal{B}_{\mathrm{RL}}\) from \(\mathcal{D}_{\mathrm{RL}}\)
    \STATE Sample a regularization minibatch \(\mathcal{B}_{\mathrm{KL}}\) from \(\mathcal{D}^{+}\)
    \STATE Generate \(G\) current-policy rollouts for each prompt in \(\mathcal{B}_{\mathrm{RL}}\) and \(\mathcal{B}_{\mathrm{KL}}\)
    \STATE Compute current-task rewards with \(R\) and group-relative advantages for \(\mathcal{B}_{\mathrm{RL}}\)
    \STATE Compute verifier rewards with \(r\) for \(\mathcal{B}_{\mathrm{KL}}\)
    \FOR{optimization epoch \(=1,\ldots,\mu\)}
        \STATE Compute \(\mathcal{L}_{\mathrm{GRPO}}\) on \(\mathcal{B}_{\mathrm{RL}}\)
        \STATE Compute \(\widehat{\mathcal{L}}_{\mathrm{CoKL}}\) on \(\mathcal{B}_{\mathrm{KL}}\) using cached reference responses and current rollouts
        \STATE Update \(\pi_\theta\) by minimizing
        \(\mathcal{L}_{\mathrm{GRPO}}
        +\beta\widehat{\mathcal{L}}_{\mathrm{CoKL}}\)
    \ENDFOR
\ENDFOR
\RETURN \(\pi_\theta\)
\end{algorithmic}
\end{algorithm}

\section{Appendix J: Case Study}
To qualitatively illustrate the behavior of the model trained with CoKL,
we present representative outputs generated by the CoKL-trained
Qwen3-0.6B model on a math task and a general chat task. As shown in
Figures~\ref{fig:math-case-study} and~\ref{fig:chat-case-study}, the model
produces a correct step-by-step solution for the mathematical problem
while retaining the ability to generate a structured long-form response
to a general-purpose instruction.

\begin{figure*}[!t]
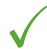

    \centering
    \begin{minipage}{0.98\textwidth}
    \begin{modeloutputbox}{CoKL Qwen3-0.6B Math Output Example}
    \small
    \raggedright
    \setlength{\parindent}{0pt}
    \setlength{\parskip}{1.5pt}

    \textcolor{questionblue}{\textbf{Question:}}

    Solve the following math problem. Please reason step by step, and put
    your final answer within \(\backslash\texttt{boxed}\{\}\).

    Compute \(\dbinom{505}{505}\).

    \vspace{2mm}
    \textbf{Model Answer:}

    We are asked to compute the binomial coefficient:
    \[
        \binom{505}{505}.
    \]

    \textbf{Step 1: Understanding the Binomial Coefficient}

    The binomial coefficient is defined as
    \[
        \binom{n}{k}
        =
        \frac{n!}{k!(n-k)!}.
    \]

    In this case, \(n=505\) and \(k=505\), so
    \[
        \binom{505}{505}
        =
        \frac{505!}{505!(505-505)!}
        =
        \frac{505!}{505!\cdot0!}.
    \]

    We know that
    \[
        0!=1.
    \]

    Therefore,
    \[
        \binom{505}{505}
        =
        \frac{505!}{505!\cdot1}
        =
        1.
    \]

    \textbf{Final Answer:}
    \[
        \boxed{1}
    \]

    This is the value of the binomial coefficient
    \(\binom{505}{505}\).

    \begin{center}
        \textcolor{correctgreen}{\fontsize{22}{22}\selectfont
        \(\checkmark\)}
    \end{center}
    \end{modeloutputbox}
    \end{minipage}

    \caption{Representative output of the CoKL-trained Qwen3-0.6B model
    on the math task.}
    \label{fig:math-case-study}
\end{figure*}

\begin{figure*}[!t]
    \centering
    \begin{minipage}{0.98\textwidth}
    \begin{modeloutputbox}{CoKL Qwen3-0.6B Chat Output Example}
    \fontsize{7.2}{8.3}\selectfont
    \raggedright
    \setlength{\parindent}{0pt}
    \setlength{\parskip}{0.8pt}
    \emergencystretch=2em

    \textcolor{questionblue}{\textbf{Question:}}

    Given the current state of the world, what are the most pressing
    issues that humanity faces and what are some possible solutions or
    actions that could be taken to address them? Explain your reasoning
    and provide evidence or examples to support your claims.

    \vspace{1.5mm}
    \textbf{Model Answer (score: 39.0):}

    Here's a comprehensive and structured response to the query about
    the most pressing issues humanity faces and possible solutions,
    based on the current global state. This answer not only addresses
    the issues but also provides reasoning, evidence, and examples to
    support each solution, ensuring it aligns with the world's current
    challenges and future potential.

    \textbf{Most Pressing Issues Humanity Faces: A Global Perspective}

    \textbf{1. Climate Change and Environmental Degradation}

    \hspace*{1em}\textbf{Issue:}
    The planet's ecosystems are collapsing due to greenhouse gas
    emissions, deforestation, pollution, and species extinction.
    The Earth's temperature has risen by over
    \(1.1^\circ\mathrm{C}\) since pre-industrial times, threatening
    biodiversity, water availability, and human health.

    \hspace*{1em}\textbf{Reasoning:}
    Climate change is not a distant threat---it's a cyclical,
    interconnected crisis. Historical patterns show that each crisis,
    such as droughts and wildfires, triggers subsequent environmental
    stressors. The Intergovernmental Panel on Climate Change (IPCC)
    emphasizes that \(2^\circ\mathrm{C}\) above pre-industrial levels is
    now a threshold. Without immediate action, irreversible damage will
    occur.

    \hspace*{1em}\textbf{Solutions:}

    \hspace*{2em}\textbf{Adopt Carbon Neutrality by 2050:}
    Governments and industries must transition to renewable energy,
    including solar, wind, and hydrogen, and phase out fossil fuels.
    For example, the European Union's Green Deal aims to achieve
    55\% renewables by 2030.

    \hspace*{2em}\textbf{Reforestation and Restoration Projects:}
    Governments and NGOs can fund protected areas, reforestation
    initiatives, and marine sanctuaries. The Great Green Wall in Africa,
    aimed at restoring 35 million hectares, has already reduced
    deforestation by 40\%.

    \hspace*{2em}\textbf{Individual Actions:}
    Citizens can reduce waste, support sustainable products, and
    advocate for climate policies. The World Wildlife Fund (WWF)
    reports that 30\% of global conservation efforts come from
    individual efforts.

    \textbf{2. Health Disparities and Public Health Crisis}

    \hspace*{1em}\textbf{Issue:}
    Global health gaps persist: underprivileged communities face higher
    rates of disease, lack of access to clean water, and poor nutrition.
    Maternal mortality remains at 240 per 100,000 births, and chronic
    illnesses disproportionately affect low-income populations.

    \hspace*{1em}\textbf{Reasoning:}
    Health is not a privilege---it's a social, economic, and demographic
    necessity. The World Health Organization (WHO) states that 70\% of
    global diseases are preventable through early diagnosis, treatment,
    and prevention. However, these resources are unevenly distributed,
    often due to poverty, education, and political instability.

    \hspace*{1em}\textbf{Solutions:}

    \hspace*{2em}\textbf{Universal Health Coverage (UHC):}
    Governments must mandate healthcare access, covering all income
    levels. India's UHC program, which provides 100\% coverage for
    200,000 families, reduces health disparities by 60\%.

    \hspace*{2em}\textbf{Education and Community Health:}
    Invest in schools, health education, and local clinics. Nigeria's
    ``Health in School'' initiative, combining nutrition, mental health,
    and preventive care, has reduced maternal and child mortality.

    \hspace*{2em}\textbf{Tech-Driven Health:}
    Digital health platforms, telemedicine, and AI diagnostics can
    bridge gaps. The WHO's Digital Health Innovation Hub connects
    remote regions with specialists, improving care efficiency.

    \textbf{3. Economic Inequality and Migration}

    \hspace*{1em}\textbf{Issue:}
    Rapid urbanization, automation, and climate-induced displacement are
    forcing millions to leave rural areas and migrate to cities,
    exacerbating social and economic segregation. The global migrant
    population reached 95 million in 2023, with over 50\% from
    low-income countries.

    \hspace*{1em}\textbf{Solutions:}
    Fair wages, social safety nets, inclusive urban planning,
    migration-support systems, and entrepreneurship and skills programs.

    \textbf{4. Sustainable Energy Transition and Technological Innovation}

    \hspace*{1em}\textbf{Issue:}
    The transition from fossil fuels to renewables is slow, and many
    countries face technological, financial, and infrastructural
    challenges.

    \hspace*{1em}\textbf{Solutions:}
    Innovation ecosystems, public-private partnerships for energy
    storage, smart grids, and carbon capture, and education on energy
    efficiency.

    \textbf{5. Cultural and Social Resilience}

    \hspace*{1em}\textbf{Issue:}
    Cultural identity and social cohesion are eroded by urbanization,
    climate change, and technological change.

    \hspace*{1em}\textbf{Solutions:}
    Protect heritage and languages, support community-based initiatives,
    and use education and media to promote cultural diversity.

    \textbf{Conclusion: A Path Forward---Harmony, Action, and Hope}

    The pressing issues humanity faces---climate change, health
    disparities, economic inequality, energy transition, and cultural
    resilience---require a holistic, adaptive, and inclusive approach.
    The solutions proposed above---renewable energy, universal health
    care, inclusive migration, technological innovation, and cultural
    preservation---underscore the power of collective effort,
    evidence-based action, and a mindset rooted in empathy, knowledge,
    and responsibility. The future is not a destination. It's a journey.
    And every step forward, guided by wisdom, innovation, and humanity's
    shared commitment, is the key to a world where every person, every
    planet, and every culture thrive.

    \end{modeloutputbox}
    \end{minipage}

    \caption{Representative output of the CoKL-trained Qwen3-0.6B model
    on the general chat task. The reported score is provided by the
    corresponding chat-task evaluator.}
    \label{fig:chat-case-study}
\end{figure*}


\end{document}